\documentclass[letterpaper]{article} 
\usepackage{aaai24}  
\usepackage{times}  
\usepackage{helvet}  
\usepackage{courier}  
\usepackage[hyphens]{url}  
\usepackage{graphicx} 
\usepackage{natbib}  
\usepackage{caption} 
\usepackage{algorithm}

\usepackage{newfloat}
\usepackage{listings}
\DeclareCaptionStyle{ruled}{labelfont=normalfont,labelsep=colon,strut=off} 
\floatstyle{ruled}
\newfloat{listing}{tb}{lst}{}
\floatname{listing}{Listing}
\nocopyright

\usepackage[T1]{fontenc}
\usepackage[utf8]{inputenc}
\usepackage{mathtools}

\usepackage{amssymb,mathrsfs}
\usepackage{amsthm}
\usepackage{bm}
\usepackage{scalerel}
\usepackage{nicefrac}
\usepackage{microtype} 
\usepackage[shortlabels]{enumitem}
\usepackage{graphicx}
\usepackage{epstopdf}
\DeclareGraphicsExtensions{.eps,.png,.jpg,.pdf}

\usepackage{url}
\usepackage{colortbl}
\usepackage{booktabs}
\usepackage{multirow}
\usepackage{colortbl,xcolor}
\usepackage{xparse,xstring}
\usepackage{calc}
\usepackage{etoolbox}

\makeatletter
\@ifpackageloaded{natbib}{
	\relax
}{
	\usepackage{cite}
}
\makeatother

\usepackage{array}
\newcolumntype{L}[1]{>{\raggedright\let\newline\\\arraybackslash\hspace{0pt}}m{#1}}
\newcolumntype{C}[1]{>{\centering\let\newline\\\arraybackslash\hspace{0pt}}m{#1}}
\newcolumntype{R}[1]{>{\raggedleft\let\newline\\\arraybackslash\hspace{0pt}}m{#1}}

\makeatletter
\let\MYcaption\@makecaption
\makeatother
\usepackage[font=footnotesize]{subcaption}
\makeatletter
\let\@makecaption\MYcaption
\makeatother

\usepackage{glossaries}
\makeatletter
\sfcode`\.1006

\let\oldgls\gls
\let\oldglspl\glspl

\newcommand\fussy@ifnextchar[3]{%
	\let\reserved@d=#1%
	\def\reserved@a{#2}%
	\def\reserved@b{#3}%
	\futurelet\@let@token\fussy@ifnch}
\def\fussy@ifnch{%
	\ifx\@let@token\reserved@d
		\let\reserved@c\reserved@a
	\else
		\let\reserved@c\reserved@b
	\fi
	\reserved@c}

\renewcommand{\gls}[1]{%
\oldgls{#1}\fussy@ifnextchar.{\@checkperiod}{\@}}
\renewcommand{\glspl}[1]{%
\oldglspl{#1}\fussy@ifnextchar.{\@checkperiod}{\@}}

\newcommand{\@checkperiod}[1]{%
	\ifnum\sfcode`\.=\spacefactor\else#1\fi
}

\robustify{\gls}
\robustify{\glspl}
\makeatother

\newacronym{wrt}{w.r.t.}{with respect to}
\newacronym{RHS}{R.H.S.}{right-hand side}
\newacronym{LHS}{L.H.S.}{left-hand side}
\newacronym{iid}{i.i.d.}{independent and identically distributed}
\newacronym{SOTA}{SOTA}{state-of-the-art}

\usepackage{float}

\usepackage[capitalize]{cleveref}
\crefname{equation}{}{}
\Crefname{equation}{}{}
\crefname{claim}{claim}{claims}
\crefname{step}{step}{steps}
\crefname{line}{line}{lines}
\crefname{condition}{condition}{conditions}
\crefname{dmath}{}{}
\crefname{dseries}{}{}
\crefname{dgroup}{}{}

\crefname{Problem}{Problem}{Problems}
\crefformat{Problem}{Problem~#2#1#3}
\crefrangeformat{Problem}{Problems~#3#1#4 to~#5#2#6}

\crefname{Theorem}{Theorem}{Theorems}
\crefname{Corollary}{Corollary}{Corollaries}
\crefname{Proposition}{Proposition}{Propositions}
\crefname{Lemma}{Lemma}{Lemmas}
\crefname{Definition}{Definition}{Definitions}
\crefname{Example}{Example}{Examples}
\crefname{Assumption}{Assumption}{Assumptions}
\crefname{Remark}{Remark}{Remarks}
\crefname{Rem}{Remark}{Remarks}
\crefname{remarks}{Remarks}{Remarks}
\crefname{Appendix}{Appendix}{Appendices}
\crefname{Supplement}{Supplement}{Supplements}
\crefname{Exercise}{Exercise}{Exercises}
\crefname{Theorem_A}{Theorem}{Theorems}
\crefname{Corollary_A}{Corollary}{Corollaries}
\crefname{Proposition_A}{Proposition}{Propositions}
\crefname{Lemma_A}{Lemma}{Lemmas}
\crefname{Definition_A}{Definition}{Definitions}

\usepackage{crossreftools}

\usepackage{algorithm}
\usepackage{algpseudocode}

\makeatletter
\def\cleartheorem#1{%
    \expandafter\let\csname#1\endcsname\relax
    \expandafter\let\csname c@#1\endcsname\relax
}
\def\clearthms#1{ \@for\tname:=#1\do{\cleartheorem\tname} }
\makeatother

\ifx\renewtheorem\undefined
	\ifx\useTheoremCounter\undefined
		\newtheorem{Theorem}{Theorem}
		\newtheorem{Corollary}{Corollary}
		\newtheorem{Proposition}{Proposition}
		\newtheorem{Lemma}{Lemma}
	\else
		\newtheorem{Theorem}{Theorem}
		
		\newtheorem{Proposition}[Theorem]{Proposition}
	\fi

	\newtheorem{Definition}{Definition}

\fi

\theoremstyle{remark}

\theoremstyle{plain}

\newcommand{\qednew}{\nobreak \ifvmode \relax \else
		\ifdim\lastskip<1.5em \hskip-\lastskip
			\hskip1.5em plus0em minus0.5em \fi \nobreak
		\vrule height0.75em width0.5em depth0.25em\fi}



\newcommand{\nn}{\nonumber\\ }

\NewDocumentCommand{\movedownsub}{e{^_}}{%
	\IfNoValueTF{#1}{%
		\IfNoValueF{#2}{^{}}
	}{%
		^{#1}
	}%
	\IfNoValueF{#2}{_{#2}}
}

\let\latexchi\chi
\RenewDocumentCommand{\chi}{}{\latexchi\movedownsub}

\newcommand{\Real}{\mathbb{R}}

\newcommand{\calF}{\mathcal{F}}

\newcommand{\calL}{\mathcal{L}}

\newcommand{\calN}{\mathcal{N}}

\newcommand{\calU}{\mathcal{U}}
\newcommand{\calV}{\mathcal{V}}

\newcommand{\bc}{\mathbf{c}}

\newcommand{\be}{\mathbf{e}}
\newcommand{\bE}{\mathbf{E}}
\newcommand{\boldf}{\mathbf{f}}
\newcommand{\bF}{\mathbf{F}}

\newcommand{\bH}{\mathbf{H}}

\newcommand{\bI}{\mathbf{I}}

\newcommand{\bk}{\mathbf{k}}

\newcommand{\bP}{\mathbf{P}}
\newcommand{\bq}{\mathbf{q}}
\newcommand{\bQ}{\mathbf{Q}}

\newcommand{\bR}{\mathbf{R}}
\newcommand{\bs}{\mathbf{s}}

\newcommand{\bt}{\mathbf{t}}

\newcommand{\bu}{\mathbf{u}}
\newcommand{\bU}{\mathbf{U}}
\newcommand{\bv}{\mathbf{v}}
\newcommand{\bV}{\mathbf{V}}

\newcommand{\bW}{\mathbf{W}}

\newcommand{\bX}{\mathbf{X}}

\newcommand{\bY}{\mathbf{Y}}

\newcommand{\bZ}{\mathbf{Z}}

\DeclareSymbolFont{bsfletters}{OT1}{cmss}{bx}{n}
\DeclareSymbolFont{ssfletters}{OT1}{cmss}{m}{n}
\DeclareMathSymbol{\bsfGamma}{0}{bsfletters}{'000}
\DeclareMathSymbol{\ssfGamma}{0}{ssfletters}{'000}
\DeclareMathSymbol{\bsfDelta}{0}{bsfletters}{'001}
\DeclareMathSymbol{\ssfDelta}{0}{ssfletters}{'001}
\DeclareMathSymbol{\bsfTheta}{0}{bsfletters}{'002}
\DeclareMathSymbol{\ssfTheta}{0}{ssfletters}{'002}
\DeclareMathSymbol{\bsfLambda}{0}{bsfletters}{'003}
\DeclareMathSymbol{\ssfLambda}{0}{ssfletters}{'003}
\DeclareMathSymbol{\bsfXi}{0}{bsfletters}{'004}
\DeclareMathSymbol{\ssfXi}{0}{ssfletters}{'004}
\DeclareMathSymbol{\bsfPi}{0}{bsfletters}{'005}
\DeclareMathSymbol{\ssfPi}{0}{ssfletters}{'005}
\DeclareMathSymbol{\bsfSigma}{0}{bsfletters}{'006}
\DeclareMathSymbol{\ssfSigma}{0}{ssfletters}{'006}
\DeclareMathSymbol{\bsfUpsilon}{0}{bsfletters}{'007}
\DeclareMathSymbol{\ssfUpsilon}{0}{ssfletters}{'007}
\DeclareMathSymbol{\bsfPhi}{0}{bsfletters}{'010}
\DeclareMathSymbol{\ssfPhi}{0}{ssfletters}{'010}
\DeclareMathSymbol{\bsfPsi}{0}{bsfletters}{'011}
\DeclareMathSymbol{\ssfPsi}{0}{ssfletters}{'011}
\DeclareMathSymbol{\bsfOmega}{0}{bsfletters}{'012}
\DeclareMathSymbol{\ssfOmega}{0}{ssfletters}{'012}

\newcommand{\btheta}{\bm{\theta}}

\makeatletter
\newcommand*\rel@kern[1]{\kern#1\dimexpr\macc@kerna}
\newcommand*\widebar[1]{%
  \begingroup
  \def\mathaccent##1##2{%
    \rel@kern{0.8}%
    \overline{\rel@kern{-0.8}\macc@nucleus\rel@kern{0.2}}%
    \rel@kern{-0.2}%
  }%
  \macc@depth\@ne
  \let\math@bgroup\@empty \let\math@egroup\macc@set@skewchar
  \mathsurround\z@ \frozen@everymath{\mathgroup\macc@group\relax}%
  \macc@set@skewchar\relax
  \let\mathaccentV\macc@nested@a
  \macc@nested@a\relax111{#1}%
  \endgroup
}
\makeatother

\DeclareMathOperator*{\argmin}{arg\,min}

\DeclareMathOperator{\var}{var}

\DeclareMathOperator{\cov}{cov}

\DeclareMathOperator*{\concat}{\scalerel*{\parallel}{\sum}}

\newcommand{\ifbcdot}[1]{\ifblank{#1}{\cdot}{#1}}

\DeclarePairedDelimiterX\abs[1]{\lvert}{\rvert}{\ifbcdot{#1}}
\DeclarePairedDelimiterX\parens[1]{(}{)}{\ifbcdot{#1}}
\DeclarePairedDelimiterX\brk[1]{[}{]}{\ifbcdot{#1}}
\DeclarePairedDelimiterX\braces[1]{\{}{\}}{\ifbcdot{#1}}
\DeclarePairedDelimiterX\angles[1]{\langle}{\rangle}{\ifblank{#1}{\cdot,\cdot}{#1}}
\DeclarePairedDelimiterX\ip[2]{\langle}{\rangle}{\ifbcdot{#1},\ifbcdot{#2}}
\DeclarePairedDelimiterX\norm[1]{\lVert}{\rVert}{\ifbcdot{#1}}
\DeclarePairedDelimiterX\ceil[1]{\lceil}{\rceil}{\ifbcdot{#1}}
\DeclarePairedDelimiterX\floor[1]{\lfloor}{\rfloor}{\ifbcdot{#1}}

\DeclareFontFamily{U}{matha}{\hyphenchar\font45}
\DeclareFontShape{U}{matha}{m}{n}{
      <5> <6> <7> <8> <9> <10> gen * matha
      <10.95> matha10 <12> <14.4> <17.28> <20.74> <24.88> matha12
      }{}
\DeclareSymbolFont{matha}{U}{matha}{m}{n}
\DeclareFontSubstitution{U}{matha}{m}{n}

\DeclareFontFamily{U}{mathx}{\hyphenchar\font45}
\DeclareFontShape{U}{mathx}{m}{n}{
      <5> <6> <7> <8> <9> <10>
      <10.95> <12> <14.4> <17.28> <20.74> <24.88>
      mathx10
      }{}
\DeclareSymbolFont{mathx}{U}{mathx}{m}{n}
\DeclareFontSubstitution{U}{mathx}{m}{n}

\DeclareMathDelimiter{\vvvert}{0}{matha}{"7E}{mathx}{"17}
\DeclarePairedDelimiterX\vertiii[1]{\vvvert}{\vvvert}{\ifbcdot{#1}}

\DeclarePairedDelimiterXPP\trace[1]{\operatorname{Tr}}{(}{)}{}{\ifbcdot{#1}} 
\DeclarePairedDelimiterXPP\col[1]{\operatorname{col}}{\{}{\}}{}{\ifbcdot{#1}} 
\DeclarePairedDelimiterXPP\row[1]{\operatorname{row}}{\{}{\}}{}{\ifbcdot{#1}} 
\DeclarePairedDelimiterXPP\erf[1]{\operatorname{erf}}{(}{)}{}{\ifbcdot{#1}}
\DeclarePairedDelimiterXPP\erfc[1]{\operatorname{erfc}}{(}{)}{}{\ifbcdot{#1}}
\DeclarePairedDelimiterXPP\KLD[2]{D}{(}{)}{}{\ifbcdot{#1}\, \delimsize\|\, \ifbcdot{#2}} 
\DeclarePairedDelimiterXPP\op[2]{\operatorname{#1}}{(}{)}{}{#2} 

\newcommand{\T}{^{\mkern-1.5mu\mathop\intercal}}

\DeclarePairedDelimiterXPP\indicate[1]{{\bf 1}}{\{}{\}}{}{\ifbcdot{#1}}

\NewDocumentCommand\ofrac{s m}{%
	\IfBooleanTF#1%
	{\dfrac{1}{#2}}%
	{\frac{1}{#2}}%
}
\NewDocumentCommand\ddfrac{s m m}{%
	\IfBooleanTF#1%
	{\dfrac{\mathrm{d} {#2}}{\mathrm{d} {#3}}}%
	{\frac{\mathrm{d} {#2}}{\mathrm{d} {#3}}}%
}
\NewDocumentCommand\ppfrac{s m m}{%
	\IfBooleanTF#1%
	{\dfrac{\partial {#2}}{\partial {#3}}}%
	{\frac{\partial {#2}}{\partial {#3}}}%
}

\providecommand\given{}

\DeclarePairedDelimiterX\Set[2]\{\}{%
\renewcommand\given{\SetSymbol[\delimsize]{#1}}
#2
}
\DeclarePairedDelimiterX\Setc[1]\{\}{%
\renewcommand\given{\SetSymbol{:}}
#1
}

\NewDocumentCommand\set{s o m}{%
	\IfBooleanTF#1%
	{\IfValueTF{#2}{\Set*{#2}{#3}}{\Setc*{#3}}}%
	{\IfValueTF{#2}{\Set{#2}{#3}}{\Setc{#3}}}%
}

\NewDocumentCommand{\evalat}{ s O{\big} m e{_^} }{%
\IfBooleanTF{#1}%
{\left. #3 \right|}{#3#2|}%
\IfValueT{#4}{_{#4}}%
\IfValueT{#5}{^{#5}}%
}

\providecommand\given{}
\DeclarePairedDelimiterXPP\cprob[1]{}(){}{
\renewcommand\given{\nonscript\,\delimsize\vert\allowbreak\nonscript\,\mathopen{}}%
#1%
}
\DeclarePairedDelimiterXPP\cexp[1]{}[]{}{
\renewcommand\given{\nonscript\,\delimsize\vert\allowbreak\nonscript\,\mathopen{}}%
#1%
}

\DeclareDocumentCommand \P { s e{_^} d() g } {%
	\mathbb{P}%
	\IfBooleanTF{#1}%
		{
			\IfValueT{#2}{_{#2}}%
			\IfValueT{#3}{^{#3}}%
			\IfValueTF{#5}{\cprob{#4 \given #5}}{\IfValueT{#4}{\cprob{#4}}}%
		}%
		{
			\IfValueT{#2}{_{#2}}%
			\IfValueT{#3}{^{#3}}%
			\IfValueTF{#5}{\cprob*{#4 \given #5}}{\IfValueT{#4}{\cprob*{#4}}}%
		}%
}

\DeclareDocumentCommand \E { s e{_^} o g } {%
	\mathbb{E}%
	\IfBooleanTF{#1}%
		{
			\IfValueT{#2}{_{#2}}%
			\IfValueT{#3}{^{#3}}%
			\IfValueTF{#5}{\cexp{#4 \given #5}}{\IfValueT{#4}{\cexp{#4}}}%
		}%
		{
			\IfValueT{#2}{_{#2}}%
			\IfValueT{#3}{^{#3}}%
			\IfValueTF{#5}{\cexp*{#4 \given #5}}{\IfValueT{#4}{\cexp*{#4}}}%
		}%
}

\DeclareDocumentCommand \Var { s e{_^} d() g } {%
	\var%
	\IfBooleanTF{#1}%
		{
			\IfValueT{#2}{_{#2}}%
			\IfValueT{#3}{^{#3}}%
			\IfValueTF{#5}{\cprob{#4 \given #5}}{\IfValueT{#4}{\cprob{#4}}}%
		}%
		{
			\IfValueT{#2}{_{#2}}%
			\IfValueT{#3}{^{#3}}%
			\IfValueTF{#5}{\cprob*{#4 \given #5}}{\IfValueT{#4}{\cprob*{#4}}}%
		}%
}

\DeclareDocumentCommand \Cov { s e{_^} d() g } {%
	\cov%
	\IfBooleanTF{#1}%
		{
			\IfValueT{#2}{_{#2}}%
			\IfValueT{#3}{^{#3}}%
			\IfValueTF{#5}{\cprob{#4 \given #5}}{\IfValueT{#4}{\cprob{#4}}}%
		}%
		{
			\IfValueT{#2}{_{#2}}%
			\IfValueT{#3}{^{#3}}%
			\IfValueTF{#5}{\cprob*{#4 \given #5}}{\IfValueT{#4}{\cprob*{#4}}}%
		}%
}

\ExplSyntaxOn
\NewDocumentCommand \dist {m o o} {%
\mathrm{#1}\left(%
	\IfValueT{#3}{%
		\tl_if_blank:nTF{ #3 }{\cdot\, \middle|\, }{#3\, \middle|\, }%
	}
	\IfValueT{#2}{#2}%
\right)%
}
\ExplSyntaxOff

\NewDocumentCommand {\cbrace} {t+ D[]{black} D(){\widthof{#5}} m m } {%
	\begingroup%
		\color{#2}
		\IfBooleanTF{#1}{%
			\overbrace{#4}^%
		}{
			\underbrace{#4}_%
		}%
		{\parbox[c]{#3}{\centering\footnotesize{#5}}}%
	\endgroup%
}

\let\oldforall\forall
\renewcommand{\forall}{\oldforall \, }

\let\oldexist\exists
\renewcommand{\exists}{\oldexist \, }

\makeatletter

\newcommand{\rankcolor}[2]{%
	\expandafter\renewcommand\csname #1\endcsname[1]{%
		\ifblank{##1}{%
			{\color{#2} \textbf{#2}}%
		}{%
			\ifmmode
				\textcolor{#2}{\bm{##1}}%
			\else%
				{\color{#2} \textbf{##1}}%
			\fi	
		}%
	}
}

\rankcolor{first}{red}
\rankcolor{second}{blue}
\rankcolor{third}{cyan}
\makeatother

\graphicspath{{./Figures/}{./figures/}}
\DeclareDocumentCommand{\includeCroppedPdf}{ o O{./Figures/} m }{
	\IfFileExists{#2#3-crop.pdf}{}{%
		\immediate\write18{pdfcrop #2#3.pdf #2#3-crop.pdf}}%
	\includegraphics[#1]{#2#3-crop.pdf}
}

\makeatletter
\newcommand*{\addFileDependency}[1]{
  \typeout{(#1)}
  \@addtofilelist{#1}
  \IfFileExists{#1}{}{\typeout{No file #1.}}
}
\makeatother

\definecolor{gray90}{gray}{0.9}
\def\colorlist{red,blue,brown,cyan,darkgray,gray,lightgray,green,lime,magenta,olive,orange,pink,purple,teal,violet,white,yellow}

\makeatletter
\def\startcomment{[}
\ifx\nohighlights\undefined
	\newcommand{\createcolor}[1]{%
			\expandafter\newcommand\csname #1\endcsname[1]{{\color{#1} ##1}}%
	}
	\newcommand{\msout}[1]{\text{\color{green} \sout{\ensuremath{#1}}}}
	\newcommand{\del}[1]{{\color{green}\ifmmode \msout{#1}\else\sout{#1}\fi}}
\else
	\newcommand{\createcolor}[1]{%
			\expandafter\newcommand\csname #1\endcsname[1]{%
				\noexpandarg%
				\StrChar{##1}{1}[\firstletter]%
				\if\firstletter\startcomment%
					\relax
				\else%
					##1
				\fi
			}%
	}
	\newcommand{\msout}[1]{}
	\newcommand{\del}[1]{}
\fi

\def\@tempa#1,{%
    \ifx\relax#1\relax\else
        \createcolor{#1}%
        \expandafter\@tempa
    \fi
}
\expandafter\@tempa\colorlist,\relax,
\makeatother

\newcommand{\hhide}[1]{}

\ifx\diagnoselabel\undefined
	\relax
\else
	\makeatletter
	\def\@testdef #1#2#3{%
		\def\reserved@a{#3}\expandafter \ifx \csname #1@#2\endcsname
			\reserved@a  \else
			\typeout{^^Jlabel #2 changed:^^J%
				\meaning\reserved@a^^J%
				\expandafter\meaning\csname #1@#2\endcsname^^J}%
			\@tempswatrue \fi}
	\makeatother
\fi

\usepackage{graphicx}
\usepackage{multirow, threeparttable, booktabs, makecell}
\usepackage{amssymb}

\newcommand{\tb}[1]{\textbf{#1}}

\title{PosDiffNet: Positional Neural Diffusion for Point Cloud Registration in a Large Field of View with Perturbations}
\author {
    Rui She\textsuperscript{\rm 1}\equalcontrib,
    Sijie Wang\textsuperscript{\rm 1}\equalcontrib,
    Qiyu Kang\textsuperscript{\rm 1}\equalcontrib\footnote{Corresponding author: Qiyu Kang.},
    Kai Zhao\textsuperscript{\rm 1}, 
    Yang Song\textsuperscript{\rm 1}, \\ 
    Wee Peng Tay\textsuperscript{\rm 1}, 
    Tianyu Geng\textsuperscript{\rm 1}, 
    Xingchao Jian\textsuperscript{\rm 1}
}
\affiliations {
    \textsuperscript{\rm 1}Nanyang Technological University, Singapore \\
    \{rui.she@, wang1679@e., qiyu.kang@, kai.zhao@, wptay@, tianyu.geng@, xingchao001@e.\}ntu.edu.sg, 
    yang.song@connect.polyu.hk
}

\usepackage{bibentry}

\begin{document}

\maketitle

\begin{abstract}
Point cloud registration is a crucial technique in 3D computer vision with a wide range of applications. However, this task can be challenging, particularly in large fields of view with dynamic objects, environmental noise, or other perturbations. To address this challenge, we propose a model called \textit{PosDiffNet}. 
Our approach performs hierarchical registration based on window-level, patch-level, and point-level correspondence. 
We leverage a graph neural partial differential equation (PDE) based on Beltrami flow to obtain high-dimensional features and position embeddings for point clouds. 
We incorporate position embeddings into a Transformer module based on a neural ordinary differential equation (ODE) to efficiently represent patches within points. 
We employ the multi-level correspondence derived from the high feature similarity scores to facilitate alignment between point clouds.
Subsequently, we use registration methods such as SVD-based algorithms to predict the transformation using corresponding point pairs. 
We evaluate PosDiffNet on several 3D point cloud datasets, verifying that it achieves state-of-the-art (SOTA) performance for point cloud registration in large fields of view with perturbations. 
The implementation code of experiments is available at \url{https://github.com/AI-IT-AVs/PosDiffNet}.
\end{abstract}

\section{Introduction}

Three-dimensional (3D) computer vision techniques recently have gained increasing popularity in various fields such as autonomous driving \cite{Wang2023HypLiLoc}, robotics \cite{li2021robotnavigation}, and scene modeling \cite{Kang22itsc}. Point cloud registration, which estimates the transformation or relative pose between two given 3D point cloud frames \cite{wang2019deep}, is a crucial task 
in many applications, such as object detection, odometry estimation, as well as simultaneous localization and mapping (SLAM) \cite{shan2020lio, Kang22itsc}, owing to its robustness against seasonal changes and illumination variations.

Iterative methods, as demonstrated by the iterative closest point (ICP) algorithm \cite{besl1992method,segal2009generalized}, have become widely employed in point cloud registration. 
Despite their utility, these methods face obstacles. Specifically, the non-convexity of the optimization problem poses a significant challenge to the attainment of a globally optimal solution \cite{wang2019deep}.
When dealing with sparse and non-uniform data, traditional methods like nearest-neighbor search may not be effective, resulting in higher registration errors \cite{wang2019deep, wei2020end}. 

To address the aforementioned challenges in point cloud registration, deep learning-based methods have been investigated to predict transformation matrices or relative poses \cite{choy2019fully, bai2020d3feat, ao2021spinnet}. 
However, achieving robust point cloud registration in large-scale scenarios remains a significant challenge due to LiDAR scan distortion and sparsity. For instance, real outdoor datasets often exhibit numerous perturbations among different frames, such as dynamic objects and environmental noise \cite{yu2019advanced}. Thus, an open question is how to efficiently estimate the transformation for large-scale scenarios with perturbations, especially in real outdoor datasets. 

In this paper, we propose a model for point cloud registration based on neural diffusion. 
Considering the demonstrated capability of Beltrami flow in preserving non-smooth graph signals and its robustness in feature representation \cite{song2022robustness}, we utilize feature descriptors and position embeddings based on graph neural diffusion with Beltrami flow \cite{kimmel1997high, chamberlain2021blend}. 
We also present a transformation estimation method using a diffusion-based Transformer. Our approach mitigates the challenges from dynamic object non-correspondence and random perturbations in large fields of view, leading to robust and efficient point cloud registration.
Our main contributions are as follows:

$\bullet$ We design a 3D point cloud representation module using graph neural diffusion based on Beltrami flow, from which point feature embedding and position embedding are both outputs. 

$\bullet$  We propose a point cloud registration method based on the window-patch-point matching and a Transformer, which incorporates neural ODE modules and leverages point features and their positional information. 

$\bullet$  We empirically evaluate our point cloud registration method to outperform other baselines in several real datasets in the large field of view with perturbations.  

\begin{figure*}[!hbt]
\centering
\includegraphics[width=\textwidth]{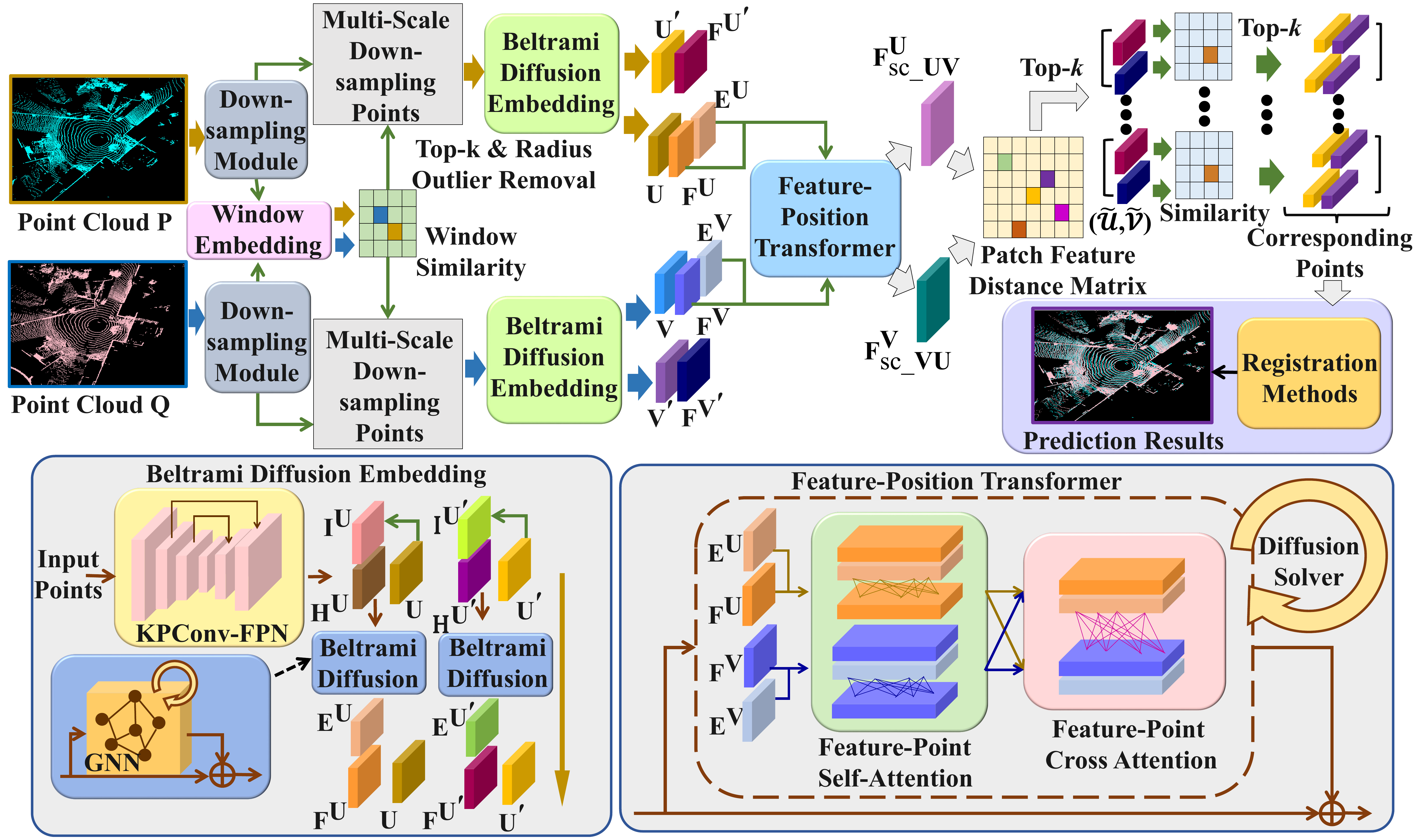}
\caption{The architecture of our PosDiffNet for the registration task \gls{wrt} point cloud pairs. Detailed information about the modules can be found in the subsequent subsections of Methodology.}
\label{fig:model_PosDiffNet} 
\end{figure*}

\section{Related Work}
\tb{Point Cloud Registration Methods.}
As classical approaches, the ICP \cite{besl1992method} and the random sample consensus (RANSAC) \cite{fischler1981random} are widely used for point cloud registration. 
RANSAC requires more computing resources and higher running time complexity due to its low convergence. ICP's performance mainly depends on the selection of the initial value. 
A series of methods to refine ICP have been proposed \cite{guerout2017mixed,koide2021voxelized} to improve the acceleration and accuracy. 
Correspondence-based estimators are used for point cloud registration. 
One type of method performs repeatable keypoint detection \cite{bai2020d3feat, huang2021predator} and then learns the keypoint descriptor for the correspondence acquisition \cite{choy2019fully, ao2021spinnet} or similarity measures to obtain the correspondences \cite{quan2020compatibility, chen2022sc2}, such as D3Feat \cite{bai2020d3feat}, SpinNet \cite{ao2021spinnet}, PREDATOR \cite{huang2021predator} and $\text{SC}^{2}$-PCR \cite{chen2022sc2}. 
The other, such as deep closest point (DCP) \cite{wang2019deep}, CoFiNet \cite{yu2021cofinet} and UDPReg \cite{mei2023unsupervised}, performs the correspondence retrieval for all possible matching point pairs without the keypoint detection. 
Additionally, auxiliary modules can be integrated into learning-based estimators, such as SuperLine3D \cite{zhao2022superline3d} and Maximal Cliques (MAC) \cite{zhang20233d}. 

In order to achieve more robust non-handcrafted estimators, learning-based methods are introduced into the transformation prediction \cite{qin2022geometric}.
Since conventional estimators like RANSAC have drawbacks in terms of convergence speed and are unstable in the presence of numerous outliers, learning-based estimators \cite{lu2021hregnet, poiesi2022learning, pais20203dregnet}, such as StickyPillars \cite{fischer2021stickypillars}, PointDSC \cite{bai2021pointdsc}, EDFNet \cite{zhang2022learning}, GeoTransformer (GeoTrans) \cite{qin2022geometric}, Lepard \cite{li2022lepard}, BUFFER \cite{ao2023buffer}, RoITr \cite{yu2023rotation} and RoReg \cite{wang2023roreg}, have attracted much interest. 

\tb{Point Cloud Feature Representation.}
In general, there are three categories of 3D feature representation methods. 
In the first category, voxel alignment is initially performed on the points, followed by the extraction of corresponding features through a 3D convolutional neural network (CNN) \cite{sindagi2019mvx, kopuklu2019resource, kumawat2019lp}. 
However, it is worth noting that this approach exhibits a long running time.
The second category focuses on the reduction of a 3D point cloud into a 2D map, subsequently leveraging classical 2D CNN techniques for feature extraction \cite{su2015multi}. 
Nevertheless, this approach may introduce unforeseen noise artifacts, which can impact the quality of the extracted features.
The third category is to extract features from the raw point clouds directly using specific neural networks, such as PointNet \cite{qi2017pointnet}, dynamic graph convolutional neural networks (DGCNN) \cite{wang2019dynamic}, point cloud transformer (PCT) \cite{guo2021pct}, GdDi \cite{poiesi2022learning}, PointMLP \cite{ma2022pointmlp}, and PointNeXt \cite{qian2022pointnext}. 

\tb{Beltrami Neural Diffusion.}
Beltrami flow is a partial differential equation widely used in signal processing \cite{kimmel1997high, chamberlain2021blend, zhao2023adversarial}. 
A Beltrami diffusion on the graph is defined \cite{song2022robustness} as
\begin{align}
\frac{\partial \bZ(\mu, t)}{\partial t} = \frac{1}{2} \frac{1}{\norm{\nabla \bZ}} \mathrm{div} \parens*{ \frac{\nabla \bZ}{\norm{\nabla \bZ}}} (\mu, t), \label{eq.beltrami_diffusion}
\end{align}
where $\mathrm{div}$ denotes the divergence, $\nabla$ denotes the gradient operator, $\norm{}$ is a norm operator, vertex feature $\bZ(\cdot, t)$ satisfies $\bZ(\mu, t)=(\bX(\mu, t), \bY(\mu, t))$ and $\mu$ is the index of vertices, 
$(\bX(\mu, t), \bY(\mu, t))$ denotes a pair of vertex features and positional features at the vertex with the index $\mu$. 
To combine the Beltrami flow and graph neural diffusion, a Beltrami neural diffusion \cite{chamberlain2021blend} is presented as  
\begin{align}
\Big[ \frac{{\rm d} \mathbf{X}(t)}{{\rm d} t}, \frac{{\rm d} \mathbf{Y}(t) }{{\rm d} t} \Big] 
& =(\mathbf{A_{B}}(\mathbf{X}(t), \mathbf{Y}(t))-\mathbf{I})[\mathbf{X}(t), \mathbf{Y}(t)], \label{eq.beltrami_GNN}\\
\mathbf{X}(0) & =\mathbf{X} ; \mathbf{Y}(0)=\alpha \mathbf{Y} ; t \geq 0, 
\end{align}
where $\mathbf{X}$ and $\mathbf{Y}$ denote the vertex feature and positional feature in a graph, respectively.
$\alpha > 0$  is a scaling factor and $\mathbf{A_{B}}(\cdot, \cdot)$ is the learnable matrix-valued function.  
From \cite{chamberlain2021grand, she2023robustmat}, most graph neural networks (GNNs) can be regarded as partial differential diffusions using different discretization, which leads that \cref{eq.beltrami_GNN} can be viewed as a neural partial differential equation (PDE).

As an advantage of neural diffusions with Beltrami flow, the robustness of feature representation for vertices is improved using both vertex features and positional features \cite{song2022robustness,chamberlain2021blend}. 
From \cref{eq.beltrami_diffusion}, since there exists a term of $\frac{1}{\norm{\nabla \bZ}}$ when the gradient is large, the feature updates slowly. This benefits the shape description for the structure of vertices \cite{song2022robustness}.  
Due to the advantages of Beltrami diffusion, this process can smooth out the noise and enhance the shape features of the input. 

\section{Methodology}\label{sect:Methodology}

\subsection{Problem Formulation}
Consider two point clouds, $\bP = \{P_i\}$ and $\bQ= \{Q_j\}$, which are subsets of $\Real^3$. We first employ a neural-diffusion-based mapping function $f$ to embed each point (or a subset of points) into a $d$-dimensional feature space, $\Real^d$. The intention behind this process is to leverage the similarities between the embeddings derived from the two point clouds for identifying matched or corresponding points. Subsequently, we anticipate predicting the rotation $\hat \bR$ and translation $\hat \bt$ that correspond to the ground-truth rotation $\bR$ and translation $\bt$.
For the point cloud registration task, our objective function is naturally defined as:
$
{\hat \bR}^*, {\hat \bt}^* = \argmin_{\hat \bR, \hat \bt} \ell_{\mathrm{D}}((\hat \bR, \hat \bt), (\bR, \bt)), 
$
where $\ell_{\mathrm{D}}$ represents a metric.
Alternatively, this can also be represented as
\begin{align}\label{eq.optimization_Rt_point}
{\hat \bR}^*, {\hat \bt}^* = \argmin_{\hat \bR, \hat \bt} \ell_{\mathrm{loss}}( \pi (\bP^{\mathrm{co}},  (\hat \bR, \hat \bt)), \bQ^{\mathrm{co}}).
\end{align}
In this context, the point-level matching $(\bP^{\mathrm{co}}, \bQ^{\mathrm{co}}) = (\{P_l\}_{l \in \calL}, \{Q_l\}_{l \in \calL})$ is established, where $P_l$ corresponds to $Q_l$, and $\calL$ is the index set of corresponding points. 
$\pi(\cdot, \cdot)$ denotes the transformation operation. 
$\ell_{\mathrm{loss}}$ is a loss function such as mean squared error (MSE).
To solve the proposed problem, we design a novel model, integrating neural diffusion. The architecture of this model is illustrated in \cref{fig:model_PosDiffNet}. 

In \cref{fig:model_PosDiffNet}, the down-sampling module with multi-scale voxel sizes is the same as that in \cite{yu2021cofinet} for obtaining window-level and patch-level central points, denoted as $(\bU_{\mathrm{win}}, \bV_{\mathrm{win}})$ and $(\bU, \bV)$ respectively.
Every window encompasses patches whose central points are within it, with each patch encapsulating points within the same process.
The window feature module consists of the DGCNN from \cite{wang2019dynamic} and the Transformer from \cite{wang2019deep}. 
The Top-$K$ and radius outlier removal methods are applied to filter out outlier windows within patches and points.
Then, we use the remaining patches and points for further registration.

\subsection{Point Cloud Representation with Beltrami Diffusion}
To represent point clouds, we initially extract both point-level and patch-level features utilizing the KPConv-FPN method \cite{qin2022geometric, thomas2019kpconv}. The two feature representations correspond to the downsampled points and patch central points. 
Then, we introduce the feature and position embeddings on the points and the patch central points.

Given a pair of original point clouds, \textit{(i)} we represent these clouds by their patch central points $(\bU, \bV)$, where $\bU \in \Real^{|\bU|\times 3}$ and $\bV \in \Real^{|\bV|\times 3}$. Each patch central point is denoted as $\bu_i$ and $\bv_j$, respectively. The learned patch-level features and position embeddings are represented by $([\bH^{\bU}, \bI^{\bU}], [\bH^{\bV}, \bI^{\bV}])$, where $\bH^{\bU}, \bI^{\bU} \in \Real^{|\bU| \times d}$ and $\bH^{\bV}, \bI^{\bV} \in \Real^{|\bV| \times d}$.
Similarly, \textit{(ii)} a pair of point clouds are denoted as $(\bU', \bV')$. The learned point-level features and position embeddings are represented by $([\bH^{\bU'}, \bI^{\bU'}], [\bH^{\bV'}, \bI^{\bV'}])$, where $\bH^{\bU'}, \bI^{\bU'} \in \Real^{|\bU'| \times d'}$ and $\bH^{\bV'}, \bI^{\bV'} \in \Real^{|\bV'| \times d'}$.
\textit{(iii)} Each patch central point can be associated with its patch consisting of points using a grouping strategy \cite{yu2021cofinet, qin2022geometric, li2018so}. The corresponding patch sets based on $(\bU, \bV)$ are denoted by 
$\calU=\{\calU_i | i=1, \ldots, |\bU|\}$ and $\calV=\{\calV_j | j=1, \ldots, |\bV|\}$, 
where $\calU_i = \{\bu'_{\eta} | \bu'_{\eta} \in \bU', \norm{\bu'_{\eta}-\bu_i} < \Gamma, \eta = 1, \ldots, |\bU'|\}$
and $\calV_j = \{\bv'_{\xi} | \bv'_{\xi} \in \bV', \norm{\bv'_{\xi}-\bv_j} < \Gamma, \xi = 1, \ldots, |\bV'|\}$, and $\Gamma$ is a threshold parameter.

To achieve enhanced robustness in the embeddings of both features and positions, we employ a neural diffusion mechanism rooted in Beltrami flow. This mechanism is specifically applied to each feature and position pair denoted as $[\bH, \bI]$.
Taking $[\bH^{\bU}, \bI^{\bU}]$ as an instance, the Beltrami neural diffusion module is characterized by
\begin{align}
\Big[ \frac{{\rm d} \bH^{\bU}(t)}{{\rm d}t}, \frac{{\rm d} \bI^{\bU}(t)}{{\rm d}t} \Big]
= f_{\mathrm{BND}}( [ \bH^{\bU}(t), \bI^{\bU}(t) ] ), \label{eq:GRAPH_pde_BND}
\end{align}
where ${f_{\mathrm{BND}}}( [\bH^{\bU}(t), \bI^{\bU}(t)]) \in \Real^{ |\bU|\times 2 d}$, $t \in [0, T_f]$, and $f_{\mathrm{BND}}(\cdot)$ denotes a Graph Neural Network (e.g. DGCNN \cite{wang2019dynamic}). This mapping is employed to embed both the point features and positions. In the context of updating the diffusion state, the construction of the neighborhood graph is derived from the $k$ nearest neighbors, based on $\bI^{\bU}(t)$ at time $t$.
The metric used to search the nearest neighbors is the $\mathcal{L}_2$ distance in the Euclidean space of point positions. 
This graph construction facilitates the effective integration of information during the diffusion process.

By integrating the \cref{eq:GRAPH_pde_BND} from $t=0$ to $t=T_f$, we obtain the embeddings for the $[\bH^{\bU}, \bI^{\bU}]$ given by 
\begin{align}
    [\bF^{\bU}, \bE^{\bU} ]= [\bH^{\bU}(T_f), \bI^{\bU}(T_f)] = \calF_{\mathrm{BND}} ([\bH^{\bU}, \bI^{\bU}]), 
\end{align}
where $\calF_{\mathrm{BND}}(\cdot)$ indicates the mapping for the Beltrami neural diffusion by solving \cref{eq:GRAPH_pde_BND}, and $\bF^{\bU}, \bE^{\bU} \in \Real^{ |\bU| \times d}$. Analogously, the embeddings corresponding to the point features and their respective positions, obtained through the Beltrami neural diffusion, are represented by $(\bF^{\bU'}, \bF^{\bV'})$. 
The architecture of the  Beltrami neural diffusion module is shown in \cref{fig:model_Beltrami}. 

\begin{figure}[!hbt]
\centering
\includegraphics[width=0.45\textwidth]{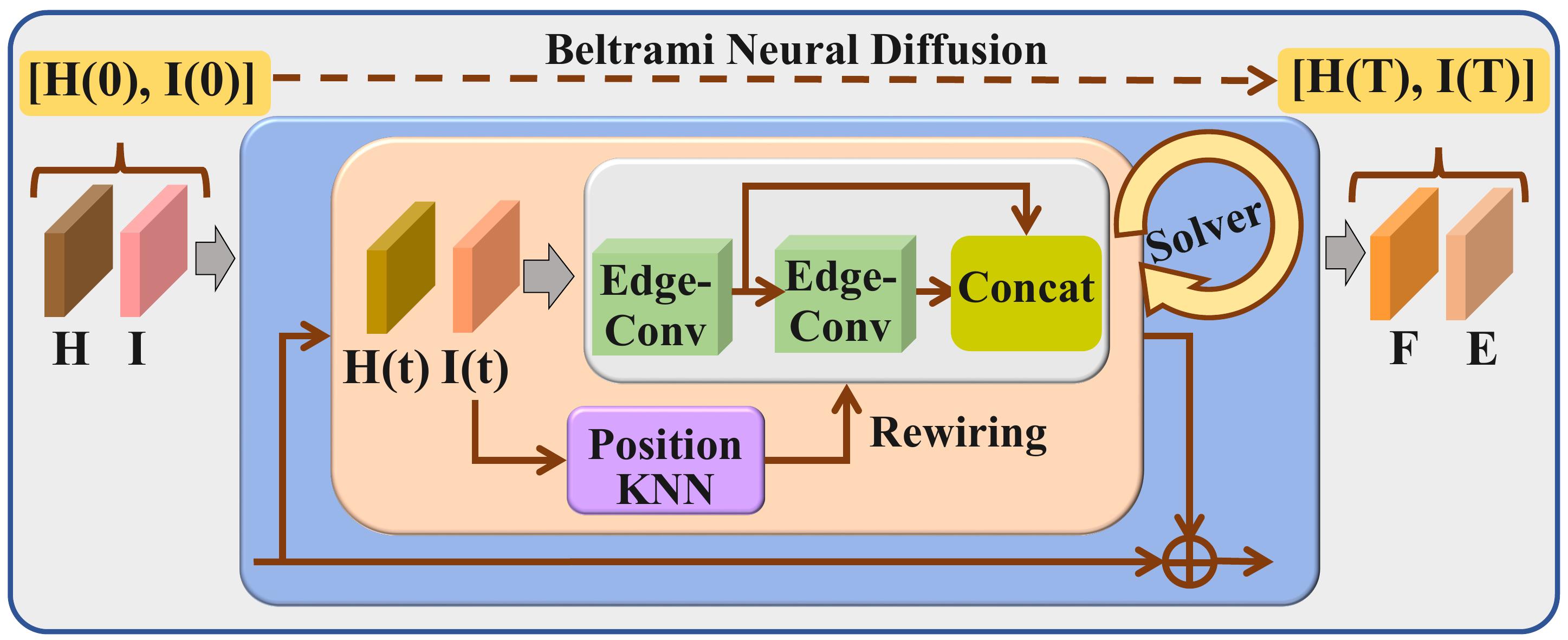}
\caption{Architecture of the Beltrami neural diffusion module for feature and position embeddings.
}
\label{fig:model_Beltrami} 
\end{figure}

\subsection{Feature-Position Transformer with Neural Diffusion}\label{sect:FP-Transformer-ND}

We propose a Transformer module based on neural ODE and the point and position embeddings derived from the Beltrami neural diffusion. 
For a pair of point clouds $(\bU,\bV)$, the input point features and position embeddings utilized as inputs for the Transformer module are represented by $[\bF^{\bU}, \bE^{\bU}]$ and $[\bF^{\bV}, \bE^{\bV}]$, respectively. Leveraging these embeddings, we proceed to elaborate on the self-attention and cross-attention mechanisms within the Transformer module.

\textbf{Feature-Position Self-Attention Mechanism.}
To emphasize the geometric position of each point and augment the richness of point representation, we integrate position embeddings into the self-attention module. For a given point cloud $\bU$, we input the normalized versions of $\bF^{\bU}$ and $\bE^{\bU}$ into the self-attention module. As a result, we obtain the embedding generated by the feature-position self-attention module as follows
\begin{align}
& f_{\mathrm{s\_att}}(\bF^{\bU})
=  \concat\limits_{i=1}^{S_{\mathrm{head}}}
\bigg(
f_{\mathrm{sfx}} \Big( \frac{(\bF^{\bU} \bW^{\mathrm{\bs\bq}}_{i}) (\bF^{\bU} \bW^{\mathrm{\bs\bk}}_{i})\T}{\sqrt{d^{\mathrm{s}}_i}} \nn
& \qquad + \frac{(\bF^{\bU} \bW^{\mathrm{\bs\be\bq}}_{i}) (\bE^{\bU} \bW^{\mathrm{\bs\be\bk}}_{i})\T}{\sqrt{d^{\mathrm{e}}_i}} \Big) ( \bF^{\bU} \bW^{\mathrm{\bs\bv}}_{i})
\bigg) {\bW}^{\mathrm{\bs}}, \label{eq.f_s_att_w}
\end{align}
where $\bW^{\mathrm{\bs\bq}}_{i}$, $\bW^{\mathrm{\bs\bk}}_{i}$, $\bW^{\mathrm{\bs\bv}}_{i}$, $\bW^{\mathrm{\bs\be\bq}}_{i}$, $\bW^{\mathrm{\bs\be\bk}}_{i}$, and ${\bW}^{\mathrm{\bs}}$ are all learnable neural networks for feature embedding. 
$d^{\mathrm{s}}_i$ and $d^{\mathrm{e}}_i$ denote the number of dimensions for point cloud features and position embeddings in the $i$-th attention head.
$(\cdot)\T$ and $\|$ are the transpose operation and the concatenation operation respectively.
$S_{\mathrm{head}}$ denotes the number of heads. $f_{\mathrm{sfx}}(\cdot)$ is the row-wise softmax normalization function. 

Furthermore, we employ the neural network module mentioned in standard Transformer architecture, including linear layers, feed forward networks (FFN), and normalization layers \cite{vaswani2017attention}, as an embedding for $f_{\mathrm{s\_att}}(\bF^{\bU})$ to obtain $f_{\mathrm{s\_ate}}(\bF^{\bU})$. 

\tb{Feature-Position Cross-Attention Mechanism.}
Based on the embeddings from the aforementioned self-attention and position information of points, we design a cross-attention for $\bU$ and $\bV$. 
When inputting the normalized $f_{\mathrm{s\_att}}(\bF^{\bU})$ \gls{wrt} $\bU$ and $f_{\mathrm{s\_ate}}(\bF^{\bV})$ \gls{wrt} $\bV$ into the cross-attention module, we have the corresponding embedding given by 
\begin{align}
& f^{\bU\bV}_{\mathrm{c\_att}}(\bF^{\bU}) \nn
& = \concat\limits_{j=1}^{C_{\mathrm{head}}}
\Bigg( f_{\mathrm{sfx}} \bigg( \frac{(f_{\mathrm{s\_att}}(\bF^{\bU}) \bW^{\mathrm{\bc\bq}}_{j})(f_{\mathrm{s\_ate}}(\bF^{\bV}) \bW^{\mathrm{\bc\bk}}_{j})\T}{\sqrt{d^{\mathrm{c}}_j}}  \nn
& 
+ \frac{(\bE^{\bU} \bW^{\mathrm{\bc\be\bq}}_{j})(\bE^{\bV} \bW^{\mathrm{\bc\be\bk}}_{j})\T}{\sqrt{d^{\mathrm{e}}_j}}
\bigg) \parens*{f_{\mathrm{s\_ate}}(\bF^{\bV}) \bW^{\mathrm{\bc\bv}}_{j}} 
\Bigg) {\bW}^{\mathrm{\bc}}, 
\end{align}
where the notations are similar to those in \cref{eq.f_s_att_w}.

Then, we combine $f_{\mathrm{s\_att}}(\bF^{\bU})$ and $ f^{\bU\bV}_{\mathrm{c\_att}}(\bF^{\bU})$ to obtain the point feature embedding for $\bU$, which is denoted by $f^{\bU\bV}_{\mathrm{sc}}(\bF^{\bU})$.
Meantime, we use fully connected (FC) layers to obtain the embedding of $\bE^{\bU}$ denoted by $f_{\mathrm{fc}}(\bE^{\bU})$. 
Furthermore, we introduce $[f^{\bU\bV}_{\mathrm{sc}}(\cdot), f_{\mathrm{fc}}(\cdot)]$ into the neural ODE to achieve the neural-diffusion-based Transformer given by 
\begin{align}
\bigg[ \frac{{\rm d} \bF^{\bU}(t)}{{\rm d} t}, \frac{{\rm d} \bE^{\bU}(t) }{{\rm d} t} \bigg] 
& = \big[ f^{\bU\bV}_{\mathrm{sc}}(\bF^{\bU}(t)), f_{\mathrm{fc}}(\bE^{\bU}(t)) \big], \label{eq.transformer_ode}
\end{align}
where $\big[ \bF^{\bU}(0), \bE^{\bU}(0) \big] = \big[ \bF^{\bU}, \bE^{\bU} \big]$ and $t \geq 0$. 
Finally, we use the output of the neural ODE, that is, the solution integrated from time $0$ to the terminal time $T$, as the embeddings $\big[ \bF^{\bU}_{\mathrm{sc\_{\bU\bV}}}$, $\bE^{\bU}_{\mathrm{sc\_{\bU\bV}}} \big]$ for $\bU$. 
Similarly, we also obtain $\big[ \bF^{\bV}_{\mathrm{sc\_{\bV\bU}}}$, $\bE^{\bV}_{\mathrm{sc\_{\bV\bU}}} \big]$ for $\bV$. These embeddings reflect the integrated information and dynamics captured by the neural ODE.
The architecture of the Transformer with neural diffusion is shown in \cref{fig:model_Transformer}.

\begin{figure}[!hbt]
\centering
\includegraphics[width=0.42\textwidth]{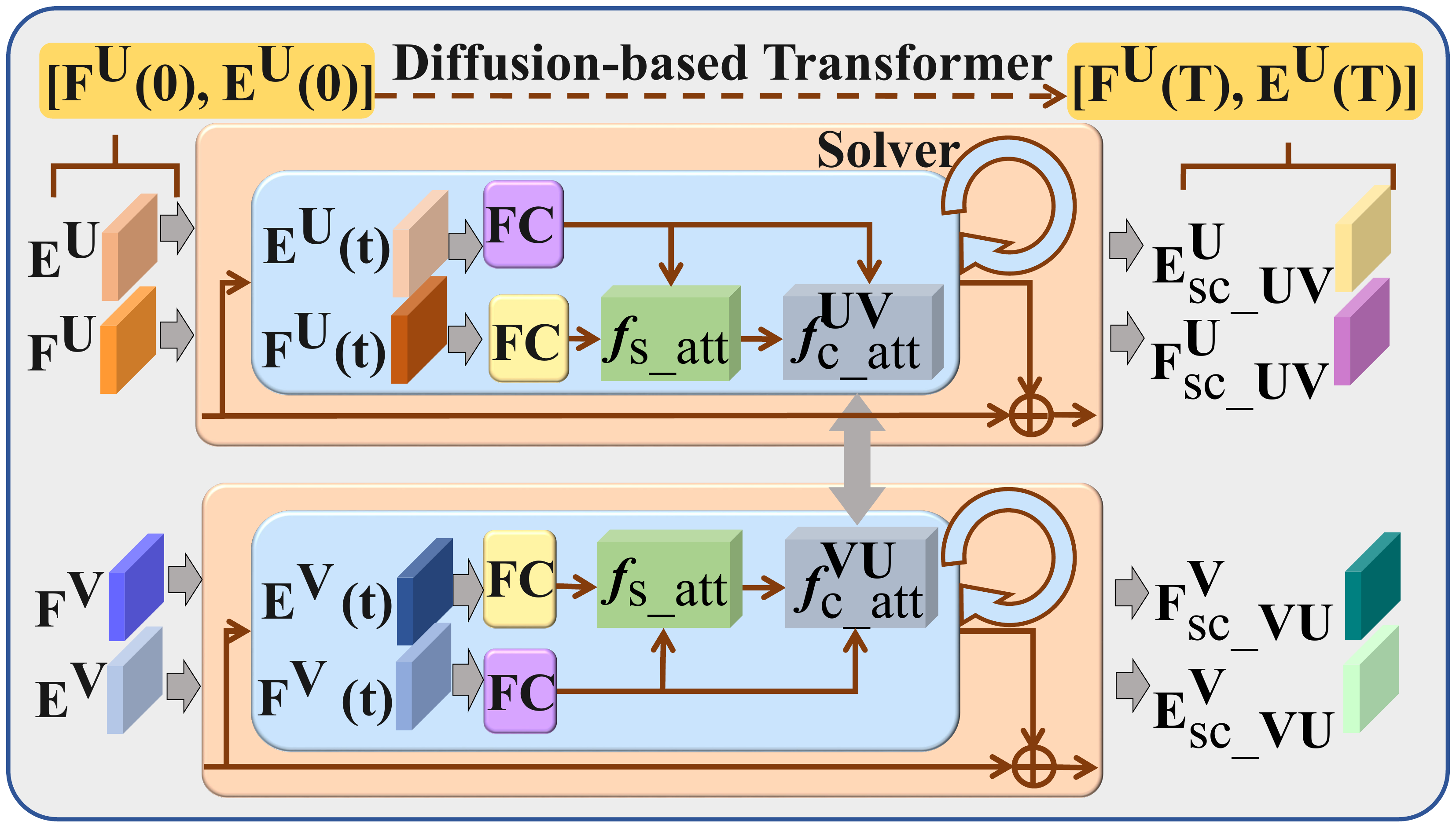}
\caption{Architecture of the feature-position Transformer based on neural ODE.}
\label{fig:model_Transformer} 
\end{figure}

\begin{table*}[!hbt] \footnotesize 
\centering
\newcommand{\tabincell}[2]{\begin{tabular}{@{}#1@{}}#2\end{tabular}}
\resizebox{\textwidth}{!}{\setlength{\tabcolsep}{7pt} 
\begin{tabular}{c | c c c c c c c c c c c c}
\hline\hline
\multirow{3}{*}{Method} & \multicolumn{4}{c}{Testing on the Boreas (Sunny)} & \multicolumn{4}{c}{Testing on the Boreas (Night)} & \multicolumn{4}{c}{Testing on the KITTI} \\
& \multicolumn{2}{c}{\textit{RTE (cm)}}  &  \multicolumn{2}{c}{\textit{RRE ($^\circ$)}}
& \multicolumn{2}{c}{\textit{RTE (cm)}}  &  \multicolumn{2}{c}{\textit{RRE ($^\circ$)}}
& \multicolumn{2}{c}{\textit{RTE (cm)}}  &  \multicolumn{2}{c}{\textit{RRE ($^\circ$)}} \\
&  \textit{MAE}     & \textit{RMSE}  & \textit{MAE} & \textit{RMSE} 
&  \textit{MAE}     & \textit{RMSE}  & \textit{MAE} & \textit{RMSE} 
&  \textit{MAE}     & \textit{RMSE}  & \textit{MAE} & \textit{RMSE}   \\ 
\hline 
ICP   
& 11.97               & 33.99                   & 0.14                & 0.35 
& 10.83               & 18.28                   & 0.11                & 0.21                
& 9.86                & 19.48                   & 0.17                & 0.27  \\ 
DCP       
& 14.59               & 25.39                   & 0.16                & 0.34 
& 11.63               & 17.36                   & 0.12                & 0.21                
& 19.96               & 31.44                   & 0.25                & 0.45  \\ 
HGNN++ 
& 13.81               & 23.63                   & 0.16                & 0.34 
& 14.41               & 23.16                   & 0.14                & 0.25                
& 10.38               & 19.69                   & 0.19                & 0.30  \\ 
VCR-Net           
& 7.51                & 16.22                   & 0.12                & 0.27 
& 8.71                & 13.56                   & 0.10                & 0.17    
& 7.62                & 14.75                   & \textbf{0.16}       & 0.25  \\ 
PCT++           
& 9.92                & 19.19                   & 0.14                & 0.30
& 9.81                & 15.77                   & 0.10                & 0.19                
& 9.40                & 17.60                   & 0.17                & 0.27  \\ 
GeoTrans  
& \textbf{3.11}       & 16.16                   & \textbf{0.08}       & 0.23   
& \underline{4.58}    & 15.78                   & \underline{0.08}    & 0.22                
& \underline{6.19}    & 10.10                   & 0.17                & 0.27   \\ 
BUFFER           
& 6.11                & \underline{7.49}        & \textbf{0.08}       & \textbf{0.12} 
& 4.64                & \underline{6.22}        & 0.11                & \underline{0.13}             
& 8.32                & \underline{9.71}        & 0.20                & \underline{0.26}   \\ 
RoITr           
& 7.66                & 13.05                    & 0.10                & 0.18 
& 9.37                & 13.68                    & 0.09                & 0.14             
& 7.50                & 11.94                    & 0.20                & 0.31   \\ \hline
PosDiffNet 
& \underline{3.38}    & \textbf{5.73}            & \textbf{0.08}       & \underline{0.15}  
& \textbf{4.46}       & \textbf{6.12}            & \textbf{0.07}       & \textbf{0.11}    
& \textbf{4.48}       & \textbf{7.28}            & \textbf{0.16}       & \textbf{0.25}    
\\
\hline\hline 
\end{tabular}}
\caption{Point cloud registration performance using the Boreas dataset for training. The best and the second-best results are highlighted in bold and underlined, respectively. 
}
\label{tab:Boreas_training}
\end{table*}

\subsection{Point Registration with Hierarchical Matching}
\tb{Hierarchical matching.} 
We conduct hierarchical matching for the corresponding windows, patches, and points. 
To match the corresponding windows and patches based on $(\bU_{\mathrm{win}}, \bV_{\mathrm{win}})$ and $(\bU, \bV)$ respectively, we conduct the exponential feature distance matrices with dual normalization \cite{sun2021loftr,qin2022geometric}. 
For instance, we perform the patch-level matching on the patch central point pairs $(\bU, \bV)$ corresponding to the point features $(\bF^{\bU}, \bF^{\bV})$.
We have the dual-normalized feature distance correlation matrix $\bW_{\bU\bV} \in \Real^{|\bU|\times|\bV|}$ where the element $w_{i,j}$ is given by 
\begin{align}\label{eq.w_matrix}
    w_{i,j}= \frac{\exp(-2\norm{\boldf^{\bU}_i-\boldf^{\bV}_j}^2_2)}
    { \sum_{j}\exp(-\norm{\boldf^{\bU}_i-\boldf^{\bV}_j}^2_2) \sum_{i}\exp(-\norm{\boldf^{\bU}_i-\boldf^{\bV}_j}^2_2)}, 
\end{align}
where $\boldf^{\bU}_i$ and $\boldf^{\bV}_j$ are the elements of $\bF^{\bU}$ and $\bF^{\bV}$ respectively. 
Then, we use Top-$K$ method to select $N_{\mathrm{p}}$ point pairs based on $w_{i,j}$, where the value of $w_{i,j}$ at the top $N_{\mathrm{p}}$-th is denoted by $w_{N_{\mathrm{p}}}$. 
We obtain the corresponding patch central points and their patches within points, respectively given by  
$(\tilde \bU, \tilde \bV) = \{(\bu_{i}, \bv_{j}) | (\bu_{i}, \bv_{j}) \in (\bU, \bV), w_{i,j} \ge w_{N_{\mathrm{p}}}, w_{i,j} \in \bW_{\bU\bV} \}$ and $(\tilde \calU, \tilde \calV) = \{(\calU_{i}, \calV_{j}) | (\calU_{i}, \calV_{j}) \in (\calU, \calV), w_{i,j} \ge w_{N_{\mathrm{p}}}, w_{i,j} \in \bW_{\bU\bV} \}$.

Furthermore, for each pair of corresponding patches within points, e.g. $(\tilde \calU_l, \tilde \calV_l) \in (\tilde \calU, \tilde \calV)$, we compute cosine similarity with post-processing Sinkhorn algorithm \cite{sarlin2020superglue} to obtain the similarity score matrix and use it to handle the point features. Then, using the Top-$K$ method, we obtain the corresponding points in this pair of patches similar to the processing in \cite{qin2022geometric, sarlin2020superglue}. 

Then, we use registration methods such as RANSAC \cite{fischler1981random}, weighted singular value decomposition (SVD) \cite{besl1992method} or local-to-global registration (LGR) \cite{qin2022geometric} to predict the rotation $\hat \bR$ and translation $\hat \bt$ based on $(\tilde{\calU}, \tilde{\calV})$. 
In this paper, the LGR method is used to achieve the point-level registration. 

\tb{Loss function.}
Due to advantages of learnable weights \cite{Wang2022robustloc,wang2020atloc}, we adopt a loss as follows
\begin{align}
    \calL = \exp(-\varpi) \calL_{\mathrm{patch}} + \varpi + \exp(-\varrho) \calL_{\mathrm{point}} + \varrho, 
\end{align}
where $\varpi$ and $\varrho$ are learnable parameters.
$\calL_{\mathrm{patch}}$ and $\calL_{\mathrm{point}}$ are the overlap-aware circle loss \cite{qin2022geometric} and negative log-likehood loss \cite{sarlin2020superglue} respectively. 

\begin{table*}[!hbt] \footnotesize 
\centering
\newcommand{\tabincell}[2]{\begin{tabular}{@{}#1@{}}#2\end{tabular}}
\resizebox{\textwidth}{!}{\setlength{\tabcolsep}{7pt} 
\begin{tabular}{c | c c c c c c c c c c c c} 
\hline\hline
\multirow{3}{*}{Method} & \multicolumn{4}{c}{Testing on the Boreas (Sunny)} & \multicolumn{4}{c}{Testing on the Boreas (Night)} & \multicolumn{4}{c}{Testing on the KITTI}  \\
& \multicolumn{2}{c}{\textit{RTE (cm)}}  &  \multicolumn{2}{c}{\textit{RRE ($^\circ$)}} 
& \multicolumn{2}{c}{\textit{RTE (cm)}}  &  \multicolumn{2}{c}{\textit{RRE ($^\circ$)}} 
& \multicolumn{2}{c}{\textit{RTE (cm)}}  &  \multicolumn{2}{c}{\textit{RRE ($^\circ$)}}  \\
&  \textit{MAE}     & \textit{RMSE}  & \textit{MAE} & \textit{RMSE}  
&  \textit{MAE}     & \textit{RMSE}  & \textit{MAE} & \textit{RMSE} 
&  \textit{MAE}     & \textit{RMSE}  & \textit{MAE} & \textit{RMSE}   \\
\hline
ICP 
& 11.97               & 33.99                   & 0.14                & 0.35
& 10.83               & 18.28                   & 0.11                & 0.21             
& 9.86                & 19.48                   & 0.17                & 0.27 \\ 
HGNN++          
& 14.31               & 24.41                   & 0.15                & 0.34
& 16.06               & 25.86                   & 0.15                & 0.27             
& 8.86                & 17.20                   & 0.20                & 0.31  \\ 
VCR-Net            
& 10.47               & 21.74                   & 0.13                & 0.29
& 11.97               & 19.78                   & 0.11                & {0.19}      
& 5.31                & {11.07}                 & \underline{0.16}    & \underline{0.24} \\ 
PCT++          

& 13.22               & 24.41                   & 0.15                & 0.34
& 11.61               & {19.57}                 & 0.13                & 0.31              
& 6.16                & 13.96                   & 0.18                & 0.28   \\ 
GeoTrans    

& \textbf{3.98}       & 18.09                   & \underline{0.09}    & 0.27
& \textbf{5.97}       & 27.90                   & \underline{0.09}    & 0.33               
& \textbf{3.93}       & 13.50                   & 0.18                & 0.50     \\ 
BUFFER           
& 6.75                & \underline{8.23}        & 0.10                & \textbf{0.12} 
& 9.17                & \underline{10.86}       & 0.12                & \underline{0.14}             
& 5.38                & \textbf{5.76}           & 0.18                & 0.29   \\ 
RoITr           
& 7.73               & 13.34                    & 0.10                & 0.18             
& 9.61               & 14.00                    & \underline{0.09}    & \underline{0.14}   
& 6.97               & 11.48                    & 0.19                & 0.29   \\  \hline
PosDiffNet   
& \underline{4.30}    & \textbf{7.32}           & \textbf{0.08}       & \underline{0.16}             
& \underline{6.65}    & \textbf{9.47}           & \textbf{0.08}       & \textbf{0.13}   
& \underline{3.97}    & \underline{6.44}        & \textbf{0.15}       & \textbf{0.23}    \\
\hline\hline
\end{tabular}}
\caption{Point cloud registration performance using the KITTI dataset for training. 
}
\label{tab:kitti_training}
\end{table*}

\section{Experiments}
\tb{Datasets.}
The Boreas dataset \cite{burnett2022boreas}, a publicly accessible street dataset comprising LiDAR and camera data, is used in our experiments. It encompasses diverse weather conditions, such as snow, rain, and nighttime scenarios. Notably, this dataset provides meticulously post-processed ground-truth poses. Leveraging these ground-truth poses, we can readily derive the transformation matrix for each adjacent pair of LiDAR point clouds.
The KITTI dataset \cite{geiger2013kitti} is also used which includes multi-sensor data. This dataset consists of $11$ sequences capturing various street scenes, and it also offers global ground-truth poses. 
More details are provided in the supplementary material. 

\tb{Implementation Details.}
We set the dimension $d$ to $256$ in \cref{eq:GRAPH_pde_BND}. For handling the neighborhood graph of the $k$ nearest neighbors, where $k=15$, we employ the graph learning layer $f_\mathrm{BND}$ in \cref{eq:GRAPH_pde_BND} as a composition of EdgeConv layers \cite{wang2019dynamic}. 
Specifically, we utilize two EdgeConv layers, 
with hidden input and output dimensions of $[1024, 512]$ and $[1536, 512]$ respectively.
We also use the DGCNN and the Transformer based on self-cross attention whose architectures are identical to those in \cite{wang2019dynamic} and \cite{wang2019deep} to extract features of window central points. 
Regarding the Transformer modules, we employ four attention heads, each with $128$ hidden features, resulting in a total of $512$ hidden features. We adopt the Adam optimizer \cite{kingma2014adam} with a learning rate of $0.0001$. The number of training epochs is set to $50$. The model is executed on an NVIDIA RTX A5000 GPU. More details are provided in the supplementary material.

\subsection{Results and Analysis}\label{sect:results}

\tb{Performance on datasets with dynamic object perturbations.}
We assess the point cloud registration performance of PosDiffNet and compare it against various baseline methods. \emph{The training data is based on the subset of the Boreas dataset collected under sunny weather conditions.}
During the testing phase, we evaluate PosDiffNet in three distinct categories. The first and second categories are subsets of the Boreas dataset, captured under different weather conditions: sunny and night, respectively. The third category consists of a subset of the KITTI dataset, where the point clouds are collected under sunny weather conditions. 
The experimental results, presented in \cref{tab:Boreas_training}, showcase the superior performance of PosDiffNet compared to the baseline methods. PosDiffNet achieves better results across most evaluation metrics, including root mean square error (RMSE) and mean absolute error (MAE), for the relative translation error (RTE) and relative rotation error (RRE). 

Furthermore, we assess the performance and conduct a comparative analysis using the subset of KITTI dataset as the training dataset. 
During the testing phase, we employ the same three categories as mentioned in the previous subsection. 
From \cref{tab:kitti_training}, we observe that PosDiffNet consistently outperforms the other baseline methods in most cases, demonstrating its superior registration performance. 

\begin{table}[!hbt] \footnotesize 
\centering
\newcommand{\tabincell}[2]{\begin{tabular}{@{}#1@{}}#2\end{tabular}}
\resizebox{0.47\textwidth}{!}{\setlength{\tabcolsep}{6pt} 
\begin{tabular}{c | c | c c c c} 
\hline\hline
\multirow{2}{*}{Weather} & \multirow{2}{*}{Method}         
& \multicolumn{2}{c}{\tabincell{c}{RTE (cm)}}  &  \multicolumn{2}{c}{\tabincell{c}{ RRE ($^\circ$)}} \\
& &  \textit{MAE}     & \textit{RMSE}  & \textit{MAE} & \textit{RMSE}   \\
\hline
\multirow{9}{*}{\tabincell{c}{Rain}} 
& ICP               
& 11.90               & 20.57                   & 0.15                & 0.27                \\ 
& DCP             
& 10.60               & 16.00                   & 0.14                & 0.22                \\ 
& HGNN++        
& 15.02               & 25.63                   & 0.18                & 0.32                \\ 
& VCR-Net      
& 8.81                & 14.09                   & 0.13                & {0.20}       \\ 
& PCT++            
& 10.39               & 16.86                   & 0.14                & 0.24                \\ 
& GeoTrans    
& \underline{4.96}    & 16.75                   & \underline{0.10}    & 0.25                \\ 
& BUFFER  
& 8.00                 & \underline{8.36}       & 0.12                 & 0.18                 \\ 
& RoITr  
& 8.01                 & 11.53                  & 0.11                 & \underline{0.16}     \\  \hline
& PosDiffNet      
& \textbf{4.56}        & \textbf{6.26}          & \textbf{0.09}        & \textbf{0.14}  \\
\hline\hline
\multirow{9}{*}{\tabincell{c}{Snow}} 
& ICP               
& 8.27                 & 12.59                     & 0.10                 & 0.15                 \\ 
& DCP            
& 7.82                 & 11.51                     & 0.12                 & 0.19                 \\ 
& HGNN++             
& 9.53                 & 14.55                     & 0.13                 & 0.21                 \\ 
& VCR-Net         
& 5.65                 & 8.48                       & 0.09                 & 0.13        \\ 
& PCT++            
& 6.66                 & 10.20                     & 0.10                 & 0.15                 \\ 
& GeoTrans  
& \textbf{3.90}         & 11.27                    & \underline{0.08}     & 0.19                 \\ 
& BUFFER  
& 7.00                 & \underline{7.58}          & 0.09                 & \textbf{0.10}                 \\ 
& RoITr  
& 8.67                 & 12.82                     & 0.10                 & 0.15                 \\ \hline
& PosDiffNet      
& \underline{4.18}     & \textbf{5.89}             & \textbf{0.07}        & \underline{0.11}  \\
\hline\hline
\end{tabular}}
\caption{Performance on the Boreas dataset under rainy and snowy weather conditions.}
\label{tab:Boreas_Rt_rain_snow}
\end{table}

\tb{Performance on datasets under bad weather conditions.}
To assess the robustness of PosDiffNet against natural noise, we conduct experiments on the Boreas dataset under adverse weather conditions such as rain and snow.
From \cref{tab:Boreas_Rt_rain_snow}, it is evident that PosDiffNet outperforms the baseline methods across all evaluation criteria. This indicates the superior performance of PosDiffNet in challenging rainy conditions.
Similarly, from \cref{tab:Boreas_Rt_rain_snow}, we observe that PosDiffNet has lower RMSE in RTE compared to the baseline methods. These results suggest that PosDiffNet produces fewer outliers among the predicted results, further verifying its robustness in handling snowy conditions compared to the baselines.

\begin{table}[!hbt] \footnotesize 
\centering
\newcommand{\tabincell}[2]{\begin{tabular}{@{}#1@{}}#2\end{tabular}}
\begin{tabular}{c | c c c c} 
\hline\hline
\multirow{2}{*}{Method}         & \multicolumn{2}{c}{\tabincell{c}{RTE(cm)}}  &  \multicolumn{2}{c}{\tabincell{c}{RTE ($^\circ$)}} \\
                                &  \textit{MAE}     & \textit{RMSE}  & \textit{MAE} & \textit{RMSE}   \\
\hline
ICP            
& 14.97               & 26.09                   & 0.20                & 0.32                   \\ 
DCP          
& 9.97                & 15.84                   & 0.29                & 0.52                   \\ 
HGNN++          
& 10.62               & 18.76                   & 0.22                & 0.34                   \\ 
VCR-Net        
& 6.40                & {12.40}                 & \textbf{0.18}       & \textbf{0.27}           \\ 
PCT++       
& 6.85                & 14.03                   & \underline{0.20}    & \underline{0.30}          \\ 
GeoTrans  
& \underline{5.37}    & 14.43                   & 0.25                & 0.50                   \\ 
BUFFER  
& 6.12                & \underline{7.04}        & 0.23                & 0.36                 \\ 
RoITr  
& 9.79                & 14.94                    & 0.27                & 0.45                 \\  \hline
PosDiffNet       
& \textbf{4.84}       & \textbf{6.93}            & \underline{0.20}     & 0.33                 \\ 
\hline\hline
\end{tabular}
\caption{Point cloud registration performance on the KITTI dataset with additive white Gaussian noise.}
\label{tab:KITTI_Rt_noise}
\end{table}

\tb{Performance on datasets with additive white Gaussian noise.}
We evaluate the robustness of PosDiffNet under the presence of additive white Gaussian noise $\mathcal{N}(\mu=0, \sigma=0.25)$ in the KITTI dataset during the testing phase.
From \cref{tab:KITTI_Rt_noise}, we observe that PosDiffNet outperforms the other benchmark methods in terms of relative translation prediction. 
Comparing Table \ref{tab:KITTI_Rt_noise} with \cref{tab:kitti_training}, we note that PosDiffNet experiences a smaller degradation in relative rotation prediction compared to the baselines. 
These findings demonstrate the crucial role of PosDiffNet in handling additive white Gaussian noise, particularly in scenarios where accurate relative translation prediction is required.

\tb{Overlapping Discussion.} 
We conduct the experiments under lower overlapping conditions using the KITTI dataset with the 10-m frame interval between each pair of frames. 
From \cref{tab:KITTI_10m_frame}, we observe that PosDiffNet outperforms or is on par with the SOTA baselines, which evaluates the efficiency of our method. 
\begin{table}[!hbt] \footnotesize 
\centering
\newcommand{\tabincell}[2]{\begin{tabular}{@{}#1@{}}#2\end{tabular}}
\begin{tabular}{c | c c c }
\hline\hline
Method           & TE(cm)                & RE($^\circ$)                 & RR(\%)                     \\ \hline
3DFeat-Net       & 25.9                  & 0.25                         & 96.0                        \\
D3Feat           & 7.2                   & 0.30                         & \textbf{99.8}               \\
SpinNet          & 9.9                   & 0.47                         & 99.1                        \\
Predator         & \underline{6.8}       & 0.27                         & \textbf{99.8}               \\ 
CoFiNet          & 8.2                   & 0.41                         & \textbf{99.8}               \\ 
PointDSC         & 8.1                   & 0.35                         & 98.2                        \\
SC$^2$-PCR       & 7.2                   & 0.32                         & 99.6                        \\
GeoTrans         & \underline{6.8}       & \textbf{0.24}                & \textbf{99.8}               \\
MAC              & 8.5                   & 0.40                         & 99.5                        \\   \hline
DGR              & $\sim$32              & 0.37                         & 98.7                        \\  
HRegNet          & $\sim$12              & 0.29                         & 99.7                        \\
UDPReg           & $\sim$8.8             & 0.41                         & 64.6                        \\
SuperLine3D      & $\sim$8.7             & 0.59                         & 97.7                        \\  \hline
PosDiffNet       & \textbf{6.6}          & \textbf{0.24}                & \textbf{99.8}                         \\
\hline\hline 
\end{tabular}
\caption{Performance on the 10-m frame KITTI dataset (the same setting as that in \cite{qin2022geometric,zhang20233d}). The results of baselines are borrowed from \cite{qin2022geometric,chen2022sc2,zhao2022superline3d,mei2023unsupervised,zhang20233d}. 
The metrics are the same as those in \cite{zhang20233d}. ``$\sim$'' indicates the lack of a dataset setting description or a setting similar to that in \cite{qin2022geometric,zhang20233d}. The PointDSC, SC$^2$-PCR, and MAC are based on the FPFH method \cite{Rusu2009fpfh}.  
}
\label{tab:KITTI_10m_frame}
\end{table}

\begin{figure}[!hbt]
\centering
\includegraphics[width=0.46\textwidth]{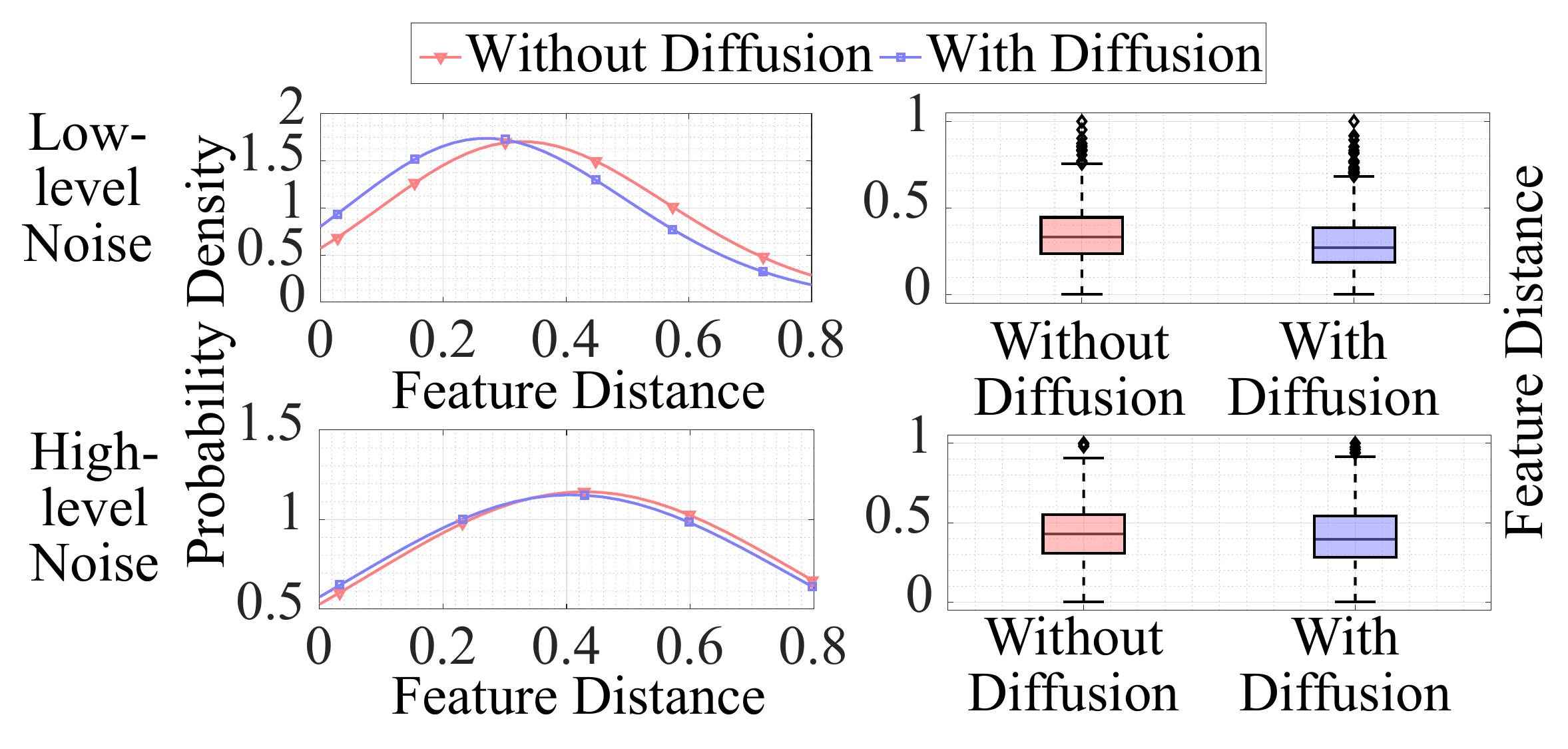}
\caption{Kernel density estimate plots and box plots for the normalized feature distance between noisy and clean conditions for the modules with or without Beltrami diffusion.
The additive noises include two Gaussian noises following $\calN(0, \sigma = 0.25)$ and $\calN(0, \sigma = 1.5)$, corresponding to the low-level and high-level noises.}
\label{fig:Beltrami_feature_dis_noise} 
\end{figure}

\begin{figure}[!hbt]
\centering
\includegraphics[width=0.46\textwidth]{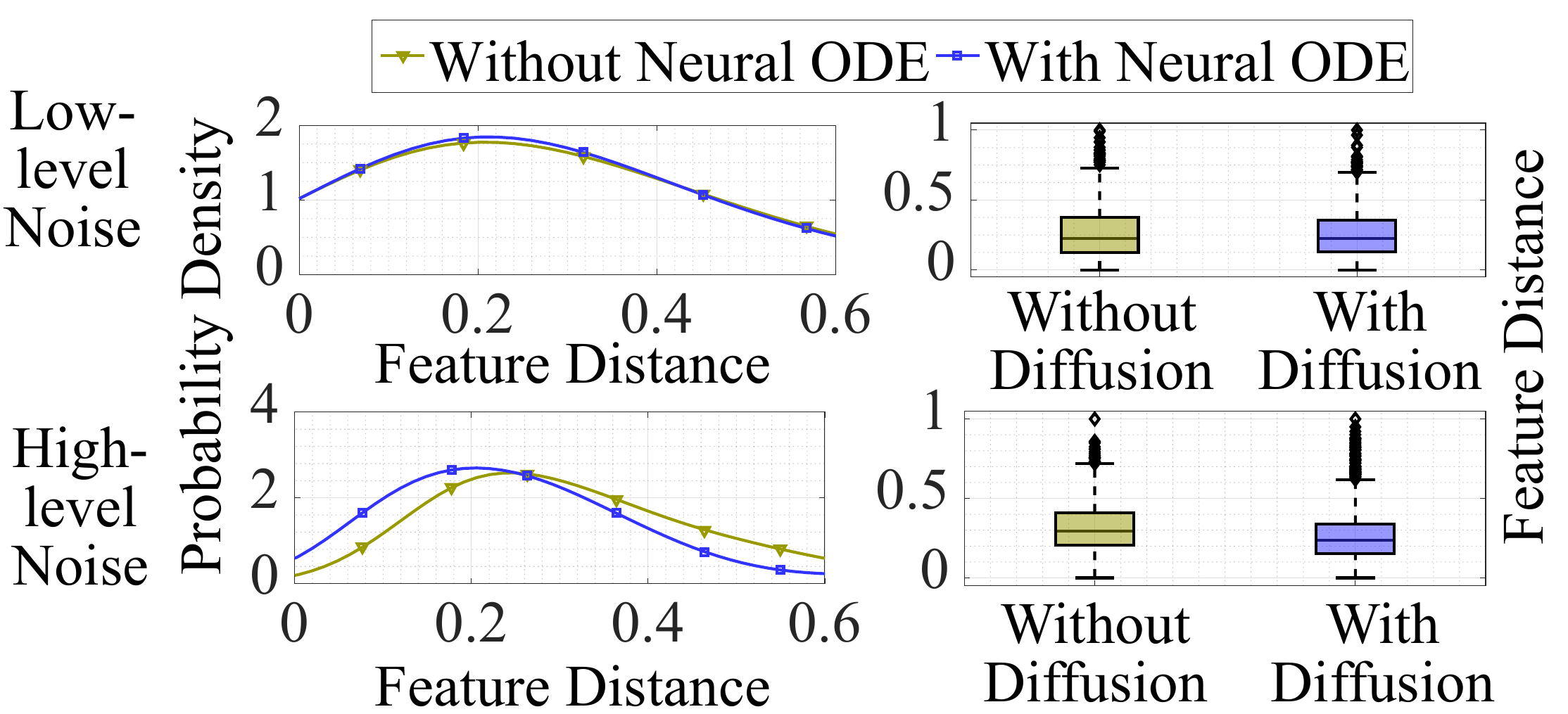}
\caption{Kernel density estimate plots and box plots for the normalized feature distance between noisy and clean conditions for the Transformer with or without neural ODE.}
\label{fig:ODE_Transformer_feature_dis_noise} 
\end{figure}

\tb{Robustness of diffusion modules.} 
We compare the feature distances based on $\calL_2$ distance between different levels of noisy conditions and the clean condition for the output of the modules with and without diffusion. From \cref{fig:Beltrami_feature_dis_noise,fig:ODE_Transformer_feature_dis_noise}, the statistics of feature distances with diffusion demonstrate superior performance and stability compared to those without diffusion. This indicates the robustness of diffusion.

\tb{Ablation Study.}
We evaluate the effectiveness of each module and the Transformer in our design. The performance improvements are observed through the evaluation of individual modules or components of the Transformer, as shown in \cref{tab:ablation_study_modules} and \cref{tab:ablation_study_Transformer}. 
\begin{table}[!hbt] \footnotesize 
\centering
\newcommand{\tabincell}[2]{\begin{tabular}{@{}#1@{}}#2\end{tabular}}
\resizebox{0.46\textwidth}{!}{\setlength{\tabcolsep}{2pt} 
\begin{tabular}{c | c c c c} 
\hline\hline
\multirow{2}{*}{Method}         & \multicolumn{2}{c}{\tabincell{c}{RTE (cm)}}  &  \multicolumn{2}{c}{\tabincell{c}{ RRE ($^\circ$)}} \\
                                &  \textit{MAE}     & \textit{RMSE}  & \textit{MAE} & \textit{RMSE}   \\
\hline
full model (with all modules)        & \textbf{3.97}        & \textbf{6.44}            & \textbf{0.15}   & \textbf{0.23}        \\  
w/o window                           & 4.34                 & 6.99                     & \textbf{0.15}   & \textbf{0.23}        \\ 
w/o Beltrami diffusion               & 5.20                 & 8.21                     & 0.16            & 0.25                 \\
w/o diffusion-based Transformer      & 11.52                & 18.72                    & 0.25            & 0.42                 \\ 
w/ Geometric Transformer              & 4.23                 & 6.89                     & \textbf{0.15}   & 0.24                 \\
\hline\hline
\end{tabular}}
\caption{Ablation for the modules.}
\label{tab:ablation_study_modules}
\end{table}
\begin{table}[!hbt] \footnotesize 
\centering
\newcommand{\tabincell}[2]{\begin{tabular}{@{}#1@{}}#2\end{tabular}}
\resizebox{0.465\textwidth}{!}{\setlength{\tabcolsep}{1pt} 
\begin{tabular}{c | c c c c} 
\hline\hline
\multirow{2}{*}{Method}         & \multicolumn{2}{c}{\tabincell{c}{RTE (cm)}}  &  \multicolumn{2}{c}{\tabincell{c}{ RRE ($^\circ$)}} \\
                                &  \textit{MAE}     & \textit{RMSE}  & \textit{MAE} & \textit{RMSE}   \\
\hline
full Transformer (with all components)  & \textbf{3.97}        & \textbf{6.44}            & \textbf{0.15}   & \textbf{0.23}        \\   
w/o Beltrami embedding                  & 4.32                 & 6.95                     & \textbf{0.15}   & \textbf{0.23}        \\ 
w/o neural ODE                          & 4.81                 & 7.79                     & 0.16            & 0.24        \\  
w/o neural ODE \& Beltrami               & 5.87                 & 9.69                     & 0.17            & 0.27        \\ 
w/  Geometric (positional) embedding    & 4.07                 & 6.55                     & \textbf{0.15}   & \textbf{0.23}         \\
\hline\hline
\end{tabular}}
\caption{Ablation for the Transformer module.}
\label{tab:ablation_study_Transformer}
\end{table}

\tb{Limitation Discussion.}
The resource-intensive nature of diffusion and attention computation presents challenges for resource-constrained devices. Future research will focus on exploring model miniaturization techniques to mitigate these constraints.

\section{Conclusion}

In this work, we introduce PosDiffNet, a model that combines a joint window-patch-point correspondence method with neural Beltrami flow and diffusion-based Transformer. PosDiffNet facilitates the simultaneous processing of point features and position information and achieves SOTA performance on datasets in large fields of view, demonstrating its effectiveness and robustness. 

\section{Acknowledgments}

This research is supported by the Singapore Ministry of Education Academic Research Fund Tier 2 grant MOE-T2EP20220-0002, and the National Research Foundation, Singapore and Infocomm Media Development Authority under its Future Communications Research and Development Programme. The computational work for this article was partially performed on resources of the National Supercomputing Centre, Singapore (https://www.nscc.sg). 




\bibliography{aaai24}

\newpage

\ \par
\newpage

\section{[Supplementary Material]}

In this supplementary material, we discuss the motivations and theoretical basis for our method.
We also provide more details about the datasets, model implementation, and baselines used in our main paper.
Then, we present additional experiments and ablation studies that are not included in the main paper due to space constraints.
Furthermore, we offer further analysis of the experimental results.
Finally, we provide point cloud alignment results as visualizations.

\section{Motivations and Theoretical Basis}


\subsection{Overview of Neural Diffusion}

In terms of a neural ordinary differential equation (ODE) layer \cite{chen2018neural,pal2021opening,lehtimaki2022accelerating}, the relationship between the input and output is defined as follows 
\begin{align}
\ddfrac{\bZ(t)}{t}=h_{\btheta}(\bZ(t), t), \label{eq:NODE}
\end{align}
where $h_{\btheta}:\Real^n \times [0,T] \rightarrow \Real^n$ is the trainable layers with the parameter $\btheta$, $\bZ: [0,T] \rightarrow \Real^n$ denotes the $n$-dimensional state, and $T$ denotes the terminal time. Simply, when the system does not explicitly depend on $t$ \cite{kang2021Neurips}, it can be regarded as the time-invariant (autonomous) case $h_{\btheta}(\bZ(t), t) = h_{\btheta}(\bZ(t))$ \cite{kang2021Neurips}.
To solve \cref{eq:NODE}, the output $\bZ(T)$ is obtained by integrating $h_{\btheta}(\bZ(t), t)$ from $t=0$ to $t=T$. 
For graph-structured data, graph neural  partial differential equations (PDEs) \cite{song2022robustness, chamberlain2021grand, Wang2022robustloc} are designed based on continuous flows, which represent the graph features more efficiently and informatively \cite{song2022robustness, kang2021Neurips, yan2019robustness, dupont2019augmented}.

\subsection{Stability of Point Cloud Representation with Beltrami Diffusion}

From the perspective of dynamical physical systems, neural diffusion methods can be regarded as dynamic systems whose stability is related to the feature representation \cite{kang2021Neurips, song2022robustness, dupont2019augmented}.
The stability of the system can be used to analyze neural graph diffusion based on Beltrami flow, the details of which are introduced as follows.

It is well known that a small perturbation at the input of an unstable dynamical system will result in a significant distortion in the system's output. First, we introduce stability in dynamical physical systems and then relate it to graph neural flows. We consider the evolution of a dynamical system described by the autonomous nonlinear differential equation mentioned in \cref{eq:NODE} in the main paper. 

\textbf{Stability of Dynamical Systems.} We introduce the notion of stability from a dynamic systems perspective, which is highly related to the robustness of graph learning against node feature perturbations.
Suppose $h_{\btheta}$ (as mentioned in \cref{eq:NODE} in the main paper) has an equilibrium at $\bZ_0$ such that $h_{\btheta}\left(\bZ_0\right)=0$.
\emph{The system is considered ``stable'' if there exists an input $|\bZ(0)|< \Delta$ ($\Delta > 0$) such that the output satisfies $|\bZ(t)|< \rho$, $\forall t \geq 0$, for some constant $\rho$.}

\begin{Definition}[Lyapunov stability \cite{katok1995introduction}] \label{def.lyp}
The equilibrium point $\mathbf{Z}_0$ is Lyapunov stable if there exists $\delta > 0$ such that for any initial condition $\mathbf{Z}(0)$ satisfying $\left|\mathbf{Z}(0)-\mathbf{Z}_0\right| < \delta$, we have $\left|\mathbf{Z}(t)-\mathbf{Z}_0\right| < \epsilon$ for all $t \geq 0$ and $\forall \epsilon > 0$.
\end{Definition}
\begin{Definition}[Asymptotically stable \cite{katok1995introduction}] 
Based on \cref{def.lyp}, the equilibrium point $\mathbf{Z}_0$ is asymptotically stable if it is Lyapunov stable and $\lim_{t\to\infty} |\mathbf{Z}(t)-\mathbf{Z}_0|=0$ for some $\varepsilon > 0$ such that $|\mathbf{Z}(0)-\mathbf{Z}_0|<\varepsilon$.
\end{Definition}

For a dynamic system with Lyapunov stability, its solutions with initial points near an equilibrium point $\mathbf{Z}_0$ remain near $\mathbf{Z}_0$. Asymptotic stability indicates that not only do trajectories stay near $\mathbf{Z}_0$ (which is known as Lyapunov stability), but the trajectories also converge to $\mathbf{Z}_0$ as time approaches infinity (which is known as asymptotic stability).

\textbf{Stability Analysis for Graph Neural Flows.} 
From \cite{song2022robustness}, the Beltrami diffusion equation \cref{eq.beltrami_GNN} in the main paper is also equal to   
\begin{align}
\frac{\partial \mathbf{Z}(t)}{\partial t}=(\mathbf{A}(\mathbf{Z}(t))-\mathbf{I}) \mathbf{Z}(t), \label{eq.graph_pde_equation}
\end{align}
where $\bZ(t)$ contains vertex features $\bX(t)$ and positional embeddings $\bY(t)$, i.e., $\bZ(t)=(\bX(t), \bY(t))$, $\mathbf{A}(\mathbf{Z}(t))$ is a learnable matrix based on $\mathbf{Z}(t)$. 
This is corresponding to the heat flow using the attention weight function. 

\begin{Proposition}[Lyapunov stability of Beltrami Neural Diffusion \cite{song2022robustness}]\label{prop.diffusion_stable}
Diffusion equation \cref{eq.graph_pde_equation} is Lyapunov stable. 
\end{Proposition} 

To solve \cref{eq.graph_pde_equation} numerically, explicit schemes are designed \cite{chamberlain2021grand,chamberlain2021blend}.
A simple method is to replace the continuous time derivative $\frac{\partial}{\partial t}$ with forward time difference, which is given by 
\begin{align}
\frac{\mathbf{z}_i^{(k+1)}-\mathbf{z}_i^{(k)}}{\tau}=\sum_{j:(i, j) \in \mathcal{E}(\bY^{(k)})} a\left(\mathbf{z}_i^{(k)}, \mathbf{z}_j^{(k)}\right)\left(\mathbf{z}_j^{(k)}-\mathbf{z}_i^{(k)}\right), \label{eq.explict_pde}
\end{align}
where $k$ is a discrete-time index, corresponding to the iteration process, and $\tau$ is the time step, corresponding to the discretization process. $\mathcal{E}(\bY^{(k)})$ is the edge set of positional embeddings $\bY^{(k)}$. 
In a matrix-vector form with $\tau=1$, \cref{eq.explict_pde} can be rewritten as 
\begin{align}
\mathbf{Z}^{(k+1)}=\left(\mathbf{A}^{(k)}-\mathbf{I}\right) \mathbf{Z}^{(k)}=\mathbf{Q}^{(k)} \mathbf{Z}^{(k)}, \label{eq.graph-pde-k}
\end{align}
where $a_{i j}^{(k)}=a\left(\mathbf{z}_i^{(k)}, \mathbf{z}_j^{(k)}\right)$ and the elements of the matrix $\mathbf{Q}^{(k)}$ are given by
\begin{align}
q_{i j}^{(k)}= \begin{cases}1-\tau \sum_{l:(i, l) \in \mathcal{E}} a_{i l}^{(k)} & i=j \\ \tau a_{i j}^{(k)} & (i, j) \in \mathcal{E}\left(\mathbf{U}^{(k)}\right) \\ 0 & \text { otherwise }\end{cases}
\end{align} 
Computing the scheme \cref{eq.graph-pde-k} for many times, the solution to the diffusion equation can be computed given an initial $\mathbf{Z}^{(0)}$. It is \emph{explicit} since the update $\mathbf{Z}^{(k+1)}$ can be obtained directly using $\mathbf{Q}^{(k)}$ and $\mathbf{Z}^{(k)}$.
From \cite{song2022robustness,chamberlain2021grand,chamberlain2021blend}, \emph{the vast majority of graph neural network architectures are explicit single step schemes of the forms \cref{eq.explict_pde} or \cref{eq.graph-pde-k}}. 
Therefore, the graph neural network $f_{\mathrm{BND}}(\cdot)$ in \cref{eq:GRAPH_pde_BND} mentioned in the main paper generally can use the single step schemes of the forms \cref{eq.explict_pde} or \cref{eq.graph-pde-k}.
This indicates that the $f_{\mathrm{BND}}(\cdot)$ can be also expressed as the form of the right-hand side of \cref{eq.graph_pde_equation}. 

\begin{Proposition}
The module of point cloud representation with Beltrami flow mentioned in the main paper has stability when the $f_{\mathrm{BND}}(\cdot)$ in \cref{eq:GRAPH_pde_BND} can be expressed as the form of \cref{eq.graph_pde_equation}. 
\end{Proposition}
\begin{proof}
From \cref{prop.diffusion_stable}, it can be observed that the Beltrami neural diffusion, expressed in the form of \cref{eq.graph_pde_equation}, is stable. This stability property is also incorporated into our module for point cloud representation using Beltrami flow. As a result, it is readily seen that our module contains a stable component, which contributes to the overall stability of the module.
\end{proof}

\subsection{Stability of Feature-Position Transformer with Neural Diffusion}

Transformer with neural diffusion mentioned in the main paper is based on the neural ordinary differential equation (ODE). In order to investigate the stability of this Transformer with neural diffusion, we first discuss the stability of neural ODE in the following context. 

\begin{Lemma}[Gronwall's Inequality \cite{yan2019robustness,snow1972gronwall}]\label{thm.G_inequality}
Let $\calF: \Omega \times[0, T] \rightarrow \mathbb{R}^d$ be a continuous function, where $\Omega \subset \mathbb{R}^d$ denotes an open set. Let two independent diffusion states $\bZ_1, \bZ_2:[0, T] \rightarrow \Omega$ satisfy the initial value problems, which are given by
\begin{align}
\frac{\mathrm{d} \bZ_1(t)}{\mathrm{d} t}=\calF\left(\bZ_1(t), t\right), & \qquad \bZ_1(0)=\bZ_1, \\
\frac{\mathrm{d} \bZ_2(t)}{\mathrm{d} t}=\calF\left(\bZ_2(t), t\right), & \qquad \bZ_2(0)=\bZ_2,  
\end{align}
where $\bZ_1$ and $\bZ_2$ are the initial states of the diffusion processes. 
Assume there is a constant $M \geq 0$ such that, $\forall t \in[0, T]$, satisfying Lipschitz continuity given by 
\begin{align}
\left.\| \calF\left(\bZ_2(t), t\right)-\calF\left(\bZ_1(t), t\right)\right)\|\leq M\| \bZ_2(t)-\bZ_1(t) \|. 
\end{align}
Then, for any $t \in[0, T]$, we have  
\begin{align}
\left\|\bZ_2(t)-\bZ_1(t)\right\| \leq\left\|\bZ_2-\bZ_1\right\| \cdot \exp(M t). 
\end{align}
\end{Lemma}

In our study, we assume that graph neural diffusion processes are autonomous, i.e. $\calF(\bZ(t), t)=\calF(\bZ(t))$. We also assume that the neural networks represent the diffusion function satisfying Lipschitz continuity. This is a common assumption when using many activation functions that have been proven to exhibit this property \cite{gouk2021regularisation}. Thus, it is readily seen that the \cref{thm.G_inequality} is held for \cref{eq.transformer_ode} in the main paper, where $\bZ(t) = \big[ \bF^{\bU}(t), \bE^{\bU}(t) \big]$ and 
$\calF(\cdot) = [f^{\bU\bV}_{\mathrm{sc}}(\cdot), f_{\mathrm{fc}}(\cdot)]$. 

\begin{Lemma}[Non-intersection for ODE integral curves \cite{dupont2019augmented,yan2019robustness}]\label{thm.unique}
Let $\bZ_1(t)$ and $\bZ_2(t)$ be two solutions of the ODE in \cref{eq:NODE} of the main paper, which are corresponding to different initial conditions, i.e. $\bZ_1(0) \neq \bZ_2(0)$. Let $h_{\bm{\theta}}$ in \cref{eq:NODE} be Lipschitz continuous. Then, it holds that $\bZ_1(t) \neq \bZ_2(t), \forall t \in[0, \infty)$. 
\end{Lemma} 

The \cref{thm.unique} indicates that integral trajectories do not intersect in the state space. 
To achieve more stable neural ODE, the core is to constrain the difference between neighboring integral curves \cite{yan2019robustness}. From \cref{thm.G_inequality}, it is readily seen that the difference between two terminal states of the diffusion process is constrained by the difference between initial states with the weights based on the exponential of the Lipschitz constant. 
Therefore, it is possible to bind the output of the diffusion, i.e., the terminal states, by controlling the initial states and weights.
This implies there exists \emph{potential stability} for the neural ODEs, which is also mentioned in \cite{yan2019robustness}. 

\begin{Proposition}
The Transformer with neural diffusion provided in the main paper has potential stability. 
\end{Proposition}
\begin{proof}
The Transformer model with neural diffusion mentioned in the main paper, incorporates the module of neural ODEs with potential stability \cite{yan2019robustness}. As a result, it is readily seen that the Transformer model has the potential for stability. 
\end{proof}

\begin{figure*}[!htb]
\begin{center}
\includegraphics[width=0.85\textwidth]{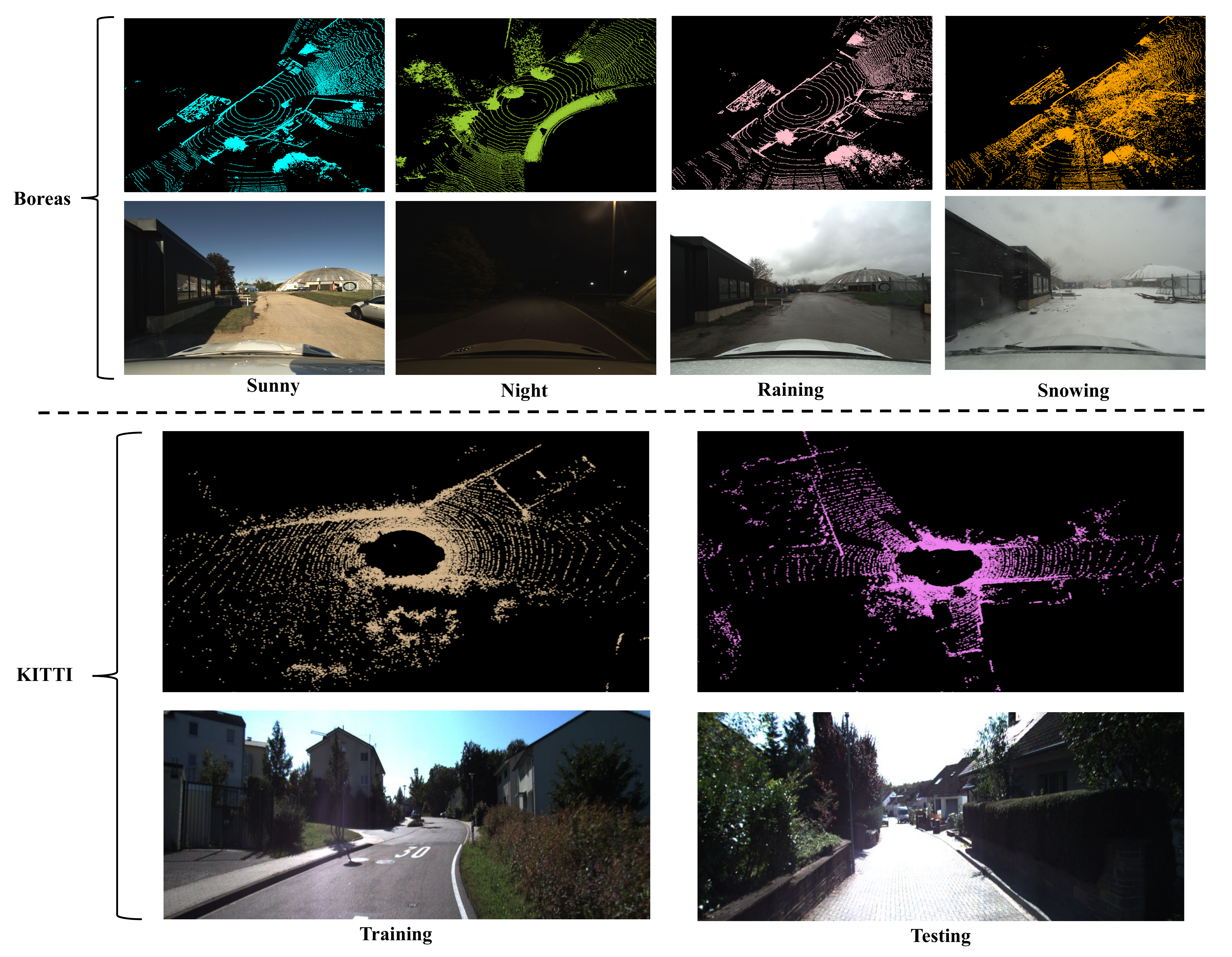}
\end{center}
\vspace{-0.5cm}
\caption{Examples of 3D LiDAR point clouds from the Boreas and KITTI datasets with their corresponding environments captured by cameras.}
\label{fig:dataset_example}
\end{figure*}

\section{Dataset and Implementation Details}\label{sect:data_model}

\subsection{Dataset Details}

\textbf{Boreas Dataset.}
The Boreas dataset is a publicly available outdoor dataset, which can be accessed at \url{https://www.boreas.utias.utoronto.ca/}.

In our study, we utilize five sequences from the Boreas dataset, including two captured under sunny weather conditions, one captured during nighttime, one captured under rainy weather conditions, and one captured under snowing weather conditions. The details and examples of the dataset used in our experiments are illustrated in \cref{stab:boreas} and \cref{fig:dataset_example}.
For each sequence, we select approximately $1500$ pairs of adjacent point clouds, with a substantial overlap between each pair. The Boreas dataset provides post-processed ground-truth poses for all sequences. To account for dynamic object perturbations within the large fields of view, we compute the transformation matrix between two consecutive LiDAR point clouds using the ground-truth poses. These ground-truth transformation matrices are used as references for the registration task.

Additionally, we introduce randomly generated transformation matrices, which serve as ground truth, for each point cloud pair. This allows us to obtain point cloud pairs with known ground-truth transformation matrices. Each synthetic transformation matrix consists of translations along the $x$, $y$, and $z$ axes, as well as rotations around the roll, pitch, and yaw axes.
The synthetic translation values are uniformly sampled from the ranges $[-1, 1]$, $[-2, 2]$, and $[-0.5, 0.5]$ along the $x$, $y$, and $z$ axes, respectively. The synthetic rotation values are uniformly sampled from the ranges $[0^\circ, 2^\circ]$, $[0^\circ, 2^\circ]$, and $[0^\circ, 15^\circ]$ around the roll, pitch, and yaw axes, respectively.
\cref{fig:dataset_example} provides some examples of the Boreas dataset.

\begin{table}[!htb] \footnotesize
\caption{Details of the Boreas dataset.}
\label{stab:boreas}
\centering
\newcommand{\tabincell}[2]{\begin{tabular}{@{}#1@{}}#2\end{tabular}}
\resizebox{0.45\textwidth}{!}{\setlength{\tabcolsep}{3pt} 
\begin{tabular}{l  c  c  c  c  c} 
\toprule
Scene & Weather Conditions & Training & Test \\
\midrule
boreas-2021-05-13-16-11 & sunny &  $\surd$ & $\times$ \\
boreas-2021-06-17-17-52 & sunny &  $\times$  &  $\surd$ \\
boreas-2021-09-14-20-00 & night &  $\times$ & $\surd$  \\
boreas-2021-04-29-15-55 & raining & $\times$ & $\surd$  \\
boreas-2021-01-26-11-22 & snowing &  $\times$ &  $\surd$ \\
\bottomrule
\end{tabular}}
\end{table}

\begin{table*}[!hbt] \footnotesize 
\caption{Point cloud registration performance using the synthetic Boreas dataset for training, where the ground truth of transformation matrices are generated randomly. The best and the second-best results are highlighted in \textbf{bold} and \underline{underlined}, respectively}
\label{tab:Boreas_training_random}
\centering
\newcommand{\tabincell}[2]{\begin{tabular}{@{}#1@{}}#2\end{tabular}}
\resizebox{\textwidth}{!}{\setlength{\tabcolsep}{7pt} 
\begin{tabular}{c c c c c c c c c c c c c}
\hline\hline
\multirow{3}{*}{Method} & \multicolumn{4}{c}{Testing on the Boreas (Sunny)} & \multicolumn{4}{c}{Testing on the Boreas (Night)} & \multicolumn{4}{c}{Testing on the KITTI} \\
& \multicolumn{2}{c}{\textit{RTE ($cm$)}}  &  \multicolumn{2}{c}{\textit{RRE ($^\circ$)}}
& \multicolumn{2}{c}{\textit{RTE ($cm$)}}  &  \multicolumn{2}{c}{\textit{RRE ($^\circ$)}}
& \multicolumn{2}{c}{\textit{RTE ($cm$)}}  &  \multicolumn{2}{c}{\textit{RRE ($^\circ$)}} \\
&  \textit{MAE}     & \textit{RMSE}  & \textit{MAE} & \textit{RMSE} 
&  \textit{MAE}     & \textit{RMSE}  & \textit{MAE} & \textit{RMSE} 
&  \textit{MAE}     & \textit{RMSE}  & \textit{MAE} & \textit{RMSE}   \\ 
\hline 
ICP 
& 2.167               & 4.546                   & 0.022               & 0.065 
& \underline{1.707}      & 3.767                   & 0.018               & 0.053             
& 4.369               & 9.677                   & 0.064               & 0.144     \\ 
DCP        
& 13.743              & 28.810                  & 0.210               & 0.539  
& 7.890               & 16.482                  & 0.027               & 0.120                
& 15.297              & 24.875                  & 0.220               & 0.438      \\ 
HGNN++ 
& 22.370              & 37.965                  & 0.079               & 0.196 
& 13.875              & 21.817                  & 0.047               & 0.062              
& 11.705              & 17.515                  & 0.102               & 0.189      \\ 
VCR-Net           
& 4.032               & 7.734                   & 0.026               & 0.068 
& 3.616               & 6.955                   & \textbf{0.004}       & 0.027    
& \underline{1.725}      & 3.410                   & 0.027               & 0.092      \\ 
PCT++             
& 2.138               & 4.030                   & 0.012               & 0.048 
& 4.190               & 11.996                  & 0.014               & 0.037             
& 1.785               & 8.568                   & \textbf{0.017}       & 0.065     \\ 
GeoTrans   
& \underline{2.041}      & \underline{3.239}          & \underline{0.008}      & \underline{0.018}   
& 2.241               & \underline{3.437}          & 0.009               & \underline{0.019}              
& 1.781               & \underline{2.676}          & 0.029               & \underline{0.064}     \\ \hline
PosDiffNet   
& \textbf{1.438}       & \textbf{2.293}           & \textbf{0.007}       & \textbf{0.015} 
& \textbf{1.437}       & \textbf{2.221}           & \underline{0.007}      & \textbf{0.017}  
& \textbf{1.384}       & \textbf{2.136}           & \underline{0.023}      & \textbf{0.048}   
\\
\hline\hline 
\end{tabular}}
\end{table*}

\textbf{KITTI Dataset.}
The KITTI dataset is a publicly available outdoor dataset, which can be accessed at \url{http://www.cvlibs.net/datasets/kitti/}.

Similar to the Boreas dataset, we utilize the provided ground-truth poses from the KITTI dataset to compute the transformation matrix between adjacent LiDAR point clouds. For our study, we select approximately $1600$ pairs for the training dataset and $1200$ pairs for the test dataset.
Additionally, we generate random transformation matrices in a similar manner as in the Boreas dataset. The synthetic translation values and rotation values are uniformly sampled from the ranges $[-1, 1]$, $[-2, 2]$, and $[-0.5, 0.5]$ along the $x$, $y$, and $z$ axes, and $[0^\circ, 2^\circ]$, $[0^\circ, 2^\circ]$, and $[0^\circ, 15^\circ]$ around the roll, pitch, and yaw axes, respectively.
Figure \ref{fig:dataset_example} showcases some examples from the KITTI dataset.

\subsection{Model Details}
We exploit the DGCNN whose architectures are identical to that in \cite{wang2019dynamic} and the Transformer based on self-cross attention the same as that in \cite{wang2019deep} to extract features of window central points. 
We use Beltrami neural diffusion modules to embed the patch-level and point-level features respectively, 
The patch-level features and point-level features are obtained from KPConv-FPN \cite{qin2022geometric, thomas2019kpconv} with $6$ encoder modules. 
Specifically, these Beltrami neural diffusion modules are based on the ``odeint'' \cite{torchdiffeq}, where the odefunction consists of $2$ EdgeConv layers \cite{wang2019dynamic} respectively, the integration time as $[0,1.0]$, the relative and absolute tolerances both as $0.01$ in the module. 
In the Transformer with neural diffusion modules, we use our feature-position Transformer in the main paper as the odefunction, and the integration time as $[0,2.0]$, the relative and absolute tolerances both as $0.01$ in the module. 
After the feature embeddings from Transformer with neural diffusion, we use weighted SVD \cite{besl1992method, qin2022geometric} to predict the transformation matrix and refine the results with local-to-global registration \cite{qin2022geometric}.  
In detail, our code is attached with the supplementary material. 
Our experiment code is developed based on the following repositories: 

\begin{itemize}
\item \url{https://github.com/twitter-research/graph-neural-pde}
\item \url{https://github.com/rtqichen/torchdiffeq}
\item \url{https://github.com/qinzheng93/GeoTransformer}
\item \url{https://github.com/magicleap/SuperGluePretrainedNetwork}
\item \url{https://github.com/huggingface/transformers}
\item \url{https://github.com/qiaozhijian/VCR-Net}
\item \url{https://github.com/iMoonLab/HGNN}
\item \url{https://github.com/Strawberry-Eat-Mango/PCT_Pytorch}
\item \url{https://github.com/WangYueFt/dcp}
\end{itemize}

\subsection{Baseline Models}

To showcase the exceptional performance of our model, we conducted a comprehensive comparison against several baseline methods, mainly including ICP \cite{besl1992method,segal2009generalized}, DCP \cite{wang2019deep}, HGNN \cite{feng2019hypergraph}, VCR-Net \cite{wei2020end}, PCT \cite{guo2021pct}, GeoTransformer (GeoTrans) \cite{qin2022geometric}, BUFFER \cite{ao2023buffer} and RoITr \cite{yu2023rotation}. The details are introduced as follows. 

As a conventional algorithm, ICP \cite{besl1992method, segal2009generalized} is an iterative optimization technique that does not rely on neural networks for feature learning, thereby obviating the need for a training process. 
In contrast, the remaining methods leverage learned point cloud features to establish point-level or path-level correspondence.
DCP \cite{wang2019deep} and VCR-Net \cite{wei2020end} use point cloud backbones to represent the features of point clouds and exploit the attention mechanisms to generate the virtual corresponding points which are used for transformation prediction. 
While HGNN \cite{feng2019hypergraph} and PCT \cite{guo2021pct} are efficient point cloud representation methods, which can be extended by incorporating attention modules to achieve point correspondence registration.
The corresponding methods are denoted as HGNN++ and PCT++ respectively. 
GeoTransformer \cite{qin2022geometric} is also a learning-based point cloud registration method, which adopts the corresponding super-point-level or patch-level correspondence to achieve the registration.  
BUFFER \cite{ao2023buffer} takes into account the accuracy, efficiency, and generalizability of registration by redesigning point-wise learners, patch-wise embedders, and inlier generators.
RoITr \cite{yu2023rotation} focuses on rotation invariance through the use of both global and local geometric information.
In addition, to compare the performance of our method and baselines, we utilize the MAE and RMSE for RRE and RTE in predictions obtained from various methods. These metrics are the same as those described in \cite{wei2020end} and \cite{wang2019deep}.
Furthermore, we include the rotation error (RE), translation error (TE), and registration recall (RR) metrics, as outlined in \cite{zhang20233d}, to evaluate the performance in the experiments conducted on the 10-m frame KITTI dataset.


\section{More Experiments and Ablation Studies}

\subsection{Performance on Synthetic Transformation.}

We evaluate the point cloud registration performance of PosDiffNet on synthetic datasets derived from subsets of the Boreas and KITTI datasets. Specifically, the training data is based on a subset of the Boreas dataset collected under sunny weather conditions. We generate point cloud pairs by randomly generating transformation matrices and using them to transform one point cloud to its corresponding point cloud.
During the testing phase, we evaluate PosDiffNet in three distinct categories, as mentioned in the main paper. We assess its performance and compare it to baseline methods using the metrics mentioned.
From \cref{tab:Boreas_training_random}, we observe that PosDiffNet almost outperforms the baseline methods across various metrics. This finding aligns with the results obtained from real transformation cases mentioned in the main paper.

\begin{table}  \footnotesize
\caption{Performance on the synthetic Boreas dataset under rainy weather conditions.}
\label{tab:Boreas_Rt_rain_random}
\centering
\newcommand{\tabincell}[2]{\begin{tabular}{@{}#1@{}}#2\end{tabular}}
\begin{tabular}{c | c c c c} 
\hline\hline
\multirow{2}{*}{Method}         & \multicolumn{2}{c}{\tabincell{c}{RTE ($cm$)}}  &  \multicolumn{2}{c}{\tabincell{c}{ RRE ($^\circ$)}} \\
                                &  \textit{MAE}     & \textit{RMSE}  & \textit{MAE} & \textit{RMSE}   \\
\hline
ICP             
& \underline{1.352}      & 2.989                   & 0.014               & 0.037                 \\ 
DCP              
& 4.354               & 8.571                   & 0.076               & 0.196                 \\ 
HGNN++           
& 16.687              & 26.790                  & 0.065               & 0.137                 \\ 
VCR-Net          
& 3.267               & 5.885                   & 0.025               & 0.065                 \\ 
PCT++           
& 1.697               & \underline{2.973}          & 0.011               & 0.026                 \\ 
GeoTrans    
& 2.015               & 3.078                   & \underline{0.009}      & \underline{0.020}                 \\ \hline
PosDiffNet       
& \textbf{1.310}       & \textbf{2.024}           & \textbf{0.007}       & \textbf{0.015}                  \\
\hline\hline
\end{tabular}
\end{table}
\begin{table} \footnotesize
\caption{Performance on the synthetic Boreas dataset under snowy conditions.}
\label{tab:Boreas_Rt_snow_random}
\centering
\newcommand{\tabincell}[2]{\begin{tabular}{@{}#1@{}}#2\end{tabular}}
\begin{tabular}{c | c c c c} 
\hline\hline
\multirow{2}{*}{Method}         & \multicolumn{2}{c}{\tabincell{c}{RTE ($cm$)}}  &  \multicolumn{2}{c}{\tabincell{c}{ RRE ($^\circ$)}} \\
                                &  \textit{MAE}     & \textit{RMSE}  & \textit{MAE} & \textit{RMSE}   \\
\hline
ICP              
& \underline{1.638}      & 8.107                   & 0.029               & 0.157                 \\ 
DCP              
& 3.676               & 6.568                   & 0.089               & 0.211                 \\ 
HGNN++           
& 12.150              & 19.082                  & 0.079               & 0.165                 \\ 
VCR-Net         
& 2.107               & 5.069                   & 0.021               & 0.141                 \\ 
PCT++         
& \textbf{1.558}       & 9.102                   & 0.027               & 0.174                 \\ 
GeoTrans    
& 2.575               & \underline{3.962}          & \underline{0.018}      & \underline{0.037}                 \\ \hline
PosDiffNet     
& 1.674               & \textbf{2.580}           & \textbf{0.011}       & \textbf{0.023}                  \\
\hline\hline
\end{tabular}
\end{table}

Furthermore, we compare PosDiffNet with baselines on the datasets under rainy and snowy weather conditions. From \cref{tab:Boreas_Rt_rain_random} and \cref{tab:Boreas_Rt_snow_random}, we also see that our PosDiffNet has advantages compared with other baselines. 

\subsection{Performance on Indoor Datasets}
we conduct the experiments on the indoor datasets, 3DMatch and 3DLoMatch in \cref{tab:3DMatch_3DLoMatch}. 
The metric is registration recall (RR) the same as that in \cite{qin2022geometric,zhang20233d}
We compare the results of more baselines, with those reported in the references. The performance of our method is comparable to the optimal results. These indoor datasets do not have noisy effects such as snow or rain, as found in the outdoor datasets. By design, our method is robust against these additive noises. Therefore, while our method is effective for indoor datasets, its performance may not match the SOTA performance for these indoor datasets. 

\begin{table}[!hbt] \footnotesize
\caption{Performance comparison on the indoor datasets including 3DMatch and 3DLoMatch. The metric is Registration Recall (RR).
The results of baselines are borrowed from \cite{yu2023rotation,mei2023unsupervised,lu2021hregnet,chen2022sc2,qin2022geometric,zhang20233d}. The symbol ``*'' is employed in MAC \cite{zhang20233d} to denote a distinct registration threshold compared to that used in most baselines. 
}
\label{tab:3DMatch_3DLoMatch}
\centering
\newcommand{\tabincell}[2]{\begin{tabular}{@{}#1@{}}#2\end{tabular}}
\resizebox{0.42\textwidth}{!}{\setlength{\tabcolsep}{14pt} 
\begin{tabular}{c | c c }
\hline\hline
\textbf{Method}    & \tabincell{c}{3DMatch \\ (RR [\%]) \\}  & \tabincell{c}{3DLoMatch \\ (RR [\%]) \\}   \\ \hline
FCGF               & 85.1                      & 40.1                                           \\                                      
D3Feat             & 81.6                      & 37.2                                           \\ 
SpinNet            & 88.6                      & 59.8                                           \\ 
Predator           & 89.0                      & 59.8                                           \\
YOHO               & 90.8                      & 65.2                                           \\ 
CoFiNet            & 89.3                      & 67.5                                           \\ 
HegNet             & 91.7                      & 55.6                                           \\
SC$^2$-PCR         & 93.3                      & 57.8                                           \\
GeoTrans           & 92.0                      & 75.0                                           \\
UDPReg             & 91.4                      & 64.3                                           \\
RoITr              & 91.9                      & 74.8                                           \\
MAC*               & 95.7                      & 78.9                                            \\ \hline
PosDiffNet         & 93.1                      & 76.0                                            \\    
\hline\hline 
\end{tabular}}
\end{table}

\subsection{Computational Complexity}
We provide the average inference time and the graphics processing unit (GPU) memory usage for each point cloud pair registration in  \cref{tab:computation_random}. The inference time is measured in seconds (s), and the GPU memory usage is measured in gigabytes (GB).
From \cref{tab:computation_random}, we observe that our PosDiffNet exhibits comparable computational complexity in terms of inference time and GPU memory usage when compared to the corresponding optimal methods. Additionally, PosDiffNet is suitable for real-time applications as its average inference time is significantly less than one second.

\begin{table}[!hbt] \footnotesize 
\caption{Average inference time and model size}
\label{tab:computation_random}
\centering
\newcommand{\tabincell}[2]{\begin{tabular}{@{}#1@{}}#2\end{tabular}}
\begin{tabular}{c | c c }
\hline\hline
Method                 & \tabincell{c}{Inference time\\}    & \tabincell{c}{GPU memory\\}         \\ \hline
ICP                    & 0.13s                              & N/A                                 \\
DCP                    & 0.09s                              & 2.71GB                              \\
HGNN++                 & 0.18s                              & 1.42GB                              \\
VCR-Net                & 0.16s                              & 1.41GB                              \\
PCT++                  & 0.62s                              & 1.99GB                              \\
GeoTrans               & 0.11s                              & 1.49GB                              \\
BUFFER                 & 0.20s                              & 1.66GB                               \\
RoITr                  & 0.11s                              & 1.55GB                               \\ \hline
PosDiffNet             & 0.12s                              & 1.91GB                               \\ 
\hline\hline 
\end{tabular}
\end{table}

\begin{table*}[!hbt] \footnotesize 
\caption{Ablation study for graph learning with Beltriami diffusion based on KITTI dataset.}
\label{tab:GNN_BND_random}
\centering
\newcommand{\tabincell}[2]{\begin{tabular}{@{}#1@{}}#2\end{tabular}}
\begin{tabular}{c | c c c c} 
\hline\hline
\multirow{2}{*}{Method}         & \multicolumn{2}{c}{\tabincell{c}{RTE ($cm$)}}  &  \multicolumn{2}{c}{\tabincell{c}{ RRE ($^\circ$)}} \\
                                &  \textit{MAE}     & \textit{RMSE}  & \textit{MAE} & \textit{RMSE}   \\
\hline
\tabincell{c}{GAT \cite{velivckovic2017graph} + Beltrami Diffusion \\} 
& 5.70            & 8.95           & 0.16               & 0.27 \\
\tabincell{c}{Self Attention \& KNN \cite{vaswani2017attention,yu2021cofinet} + Beltrami Diffusion \\}
& 4.28            & 6.99           & {0.15}          & {0.23} \\
\tabincell{c}{Point-PN \cite{zhang2023parameter} + Beltrami Diffusion}
& {3.91}       & {6.35}      & {0.15}          & 0.24 \\ \hline
\tabincell{c}{EdgeConv layers \cite{wang2019dynamic} + Beltrami Diffusion}
& 3.97            & 6.44           & {0.15}          & {0.23} \\
\hline\hline
\end{tabular}
\end{table*}

\subsection{Additional Ablation Study}

\textbf{Ablation Study for Graph Learning in the Beltrami Diffusion Module.}
We investigate the impact of graph learning in the Beltrami flow on the point cloud registration performance. In our PosDiffNet, we explore different graph learning methods by replacing the EdgeConv layers with graph attention network (GAT) \cite{velivckovic2017graph}, self-attention for $K$ nearest neighbors (KNN) \cite{vaswani2017attention, yu2021cofinet}, and Point-PN \cite{zhang2023parameter}. These methods are combined with Beltrami diffusion to achieve point feature representation and position embedding.
From \cref{tab:GNN_BND_random}, we observe that most of the methods exhibit similar performance, except for GAT. The EdgeConv layers demonstrate advantages in rotation prediction, while Point-PN outperforms the other methods in translation prediction.

\textbf{Comparison for the rewiring methonds.}
We compare the rewiring methods in the Beltrami diffusion module.
In our method, the rewiring is achieved by utilizing the positional embeddings of the 3D coordinates of points to identify neighboring points based on KNN method. While, in the original Beltrami diffusion, the positional embeddings are obtained from the high-dimensional point features, like those used for vertex embeddings. 
From \cref{tab:rewiring}, we observe that our rewiring method has advantages compared with that in the original Beltrami diffusion. 

\begin{table}[!hbt] \footnotesize
\caption{Performance comparison for difference graph rewiring of the Beltrami module.}
\label{tab:rewiring}
\centering
\newcommand{\tabincell}[2]{\begin{tabular}{@{}#1@{}}#2\end{tabular}}
\resizebox{0.48\textwidth}{!}{\setlength{\tabcolsep}{7pt} 
\begin{tabular}{c | c c}
\hline\hline
\textbf{Method}         & \tabincell{c}{PosdiffNet (with \\  original Beltrami \\ rewiring\\}    
                        & \tabincell{c}{PosdiffNet (with \\  our Beltrami \\ rewiring)   \\ }    \\ \hline
MAE of RTE (cm)         & 4.08    & 3.97                                                         \\
RMSE of RTE(cm)         & 6.74    & 6.44                                                         \\ 
MAE of RRE ($^\circ$)   & 0.15    & 0.15                                                         \\
RMSE of RRE($^\circ$)   & 0.23    & 0.23                                                         \\ 
\hline\hline
\end{tabular}}
\end{table}


\section{More Experimental Result Analysis}\label{sect:analysis}

We present additional experimental results to conduct further analysis of the performance of our PosDiffNet, focusing on the empirical probability of prediction errors. The details are provided below.

\begin{figure}[!hbt]
\centering
\includegraphics[width=0.49\textwidth]{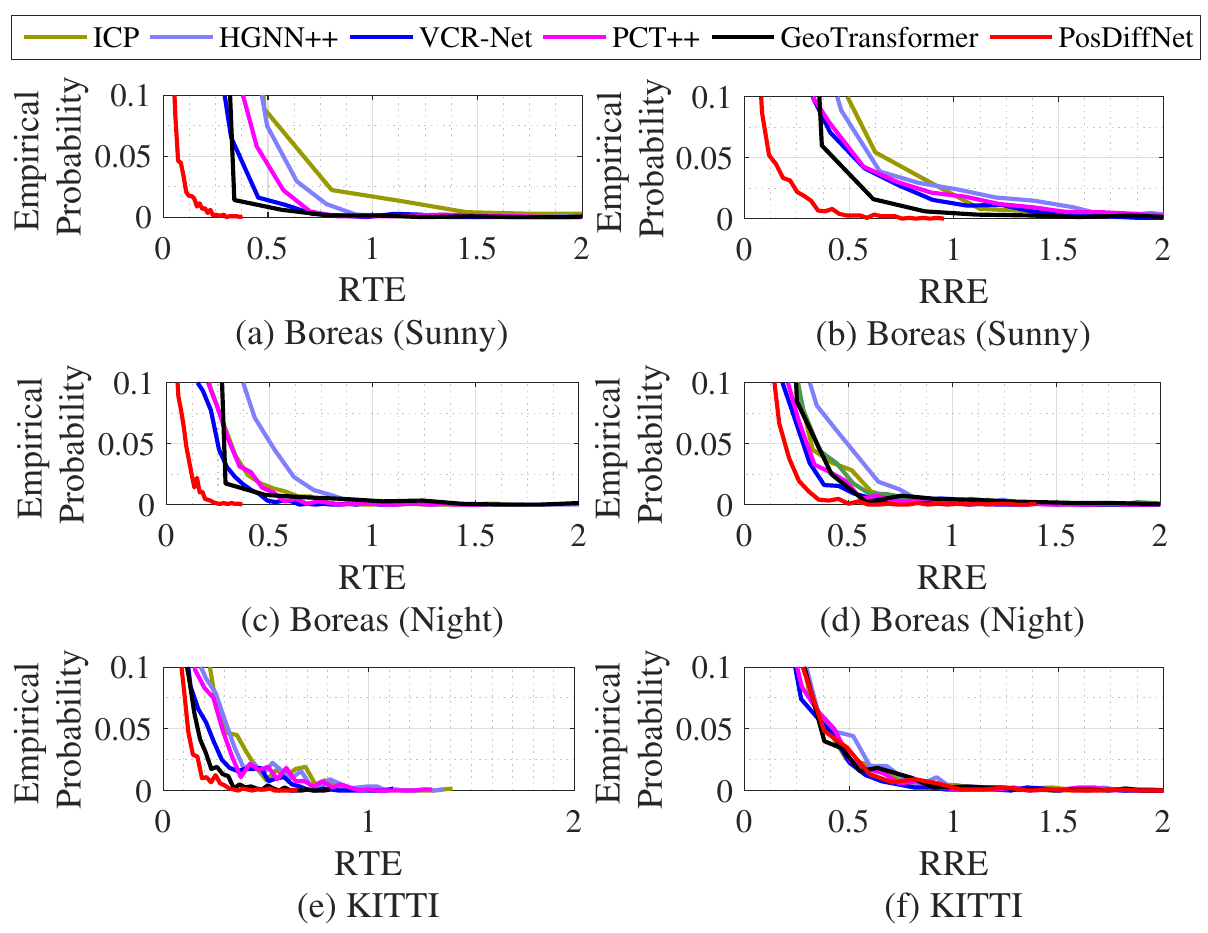}
\vspace{-0.5cm}
\caption{The empirical probability for RTE (meter [$m$]) and RRE (degree [$^{\circ}$]) using the Boreas dataset for training.}
\label{fig:pdf-error-boreas-random} 
\end{figure}


\begin{figure}[!hbt]
\centering
\includegraphics[width=0.49\textwidth]{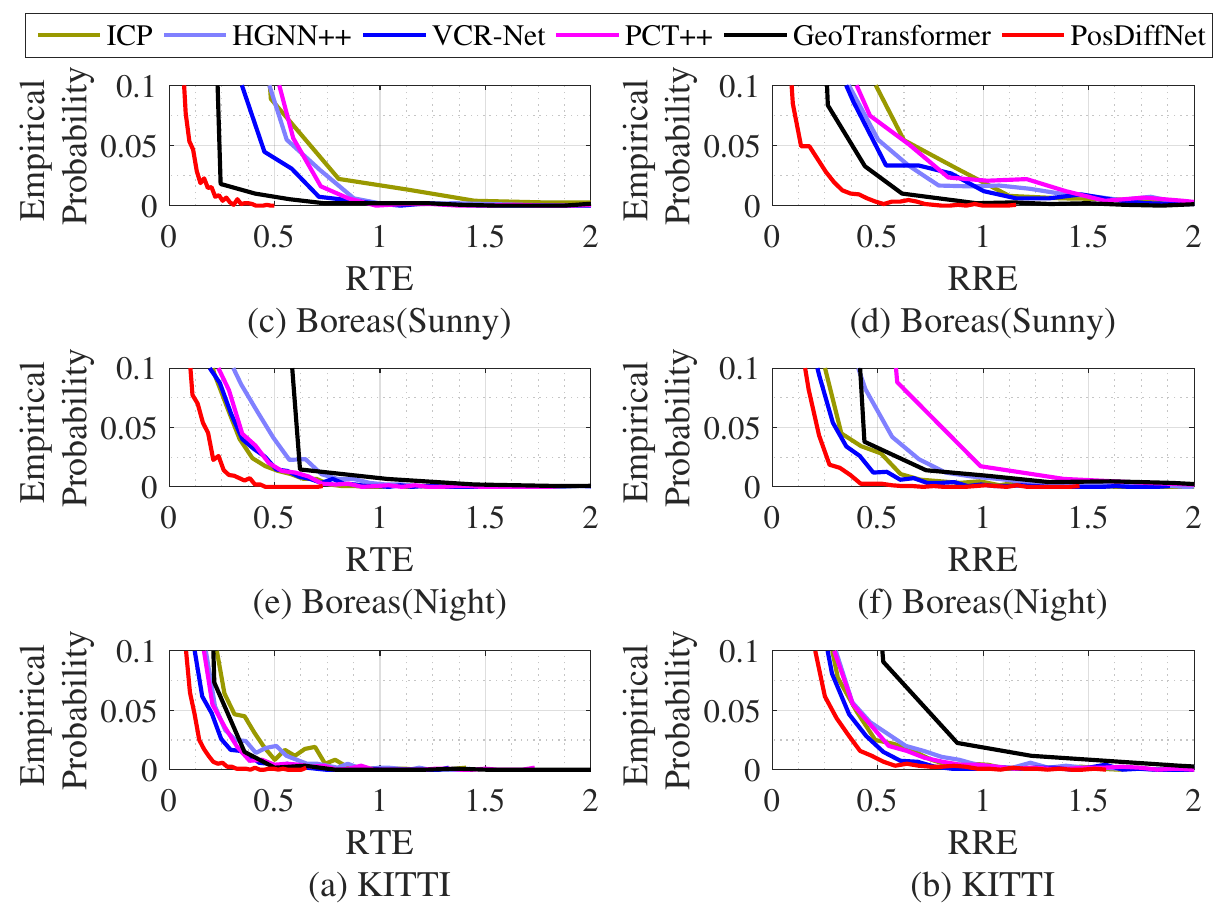}
\vspace{-0.5cm}
\caption{The empirical probability for RTE (meter [$m$]) and RRE (degree [$^{\circ}$]) using the KITTI dataset for training.}
\label{fig:pdf-error-kitti-random} 
\end{figure}


In order to show more performance from the statistic perspective, we present the empirical probability for RTE and RRE. 
This is also corresponding to the performance of Boreas and KITTI datasets mentioned in the main paper. 
From \cref{fig:pdf-error-boreas-random} and \cref{fig:pdf-error-kitti-random}, we observe that our PosDiffNet converges to probability one faster than other baselines in terms of empirical probability for RTE and RRE.
That is to say, the prediction from PosDiffNet is close to its ground truth in large probability with less prediction error occurring. 
As a result, our PosDiffNet has more accurate transformation prediction compared with the other methods. 

\section{Visualization and Examples} 
By utilizing the predicted transformation between two point cloud frames, we align the second frame with the coordinate system of the first frame, facilitating accurate alignment of the point clouds. The alignment results are shown as follows.
\begin{figure*}[!hbt]
\centering
\includegraphics[width=0.93\textwidth]{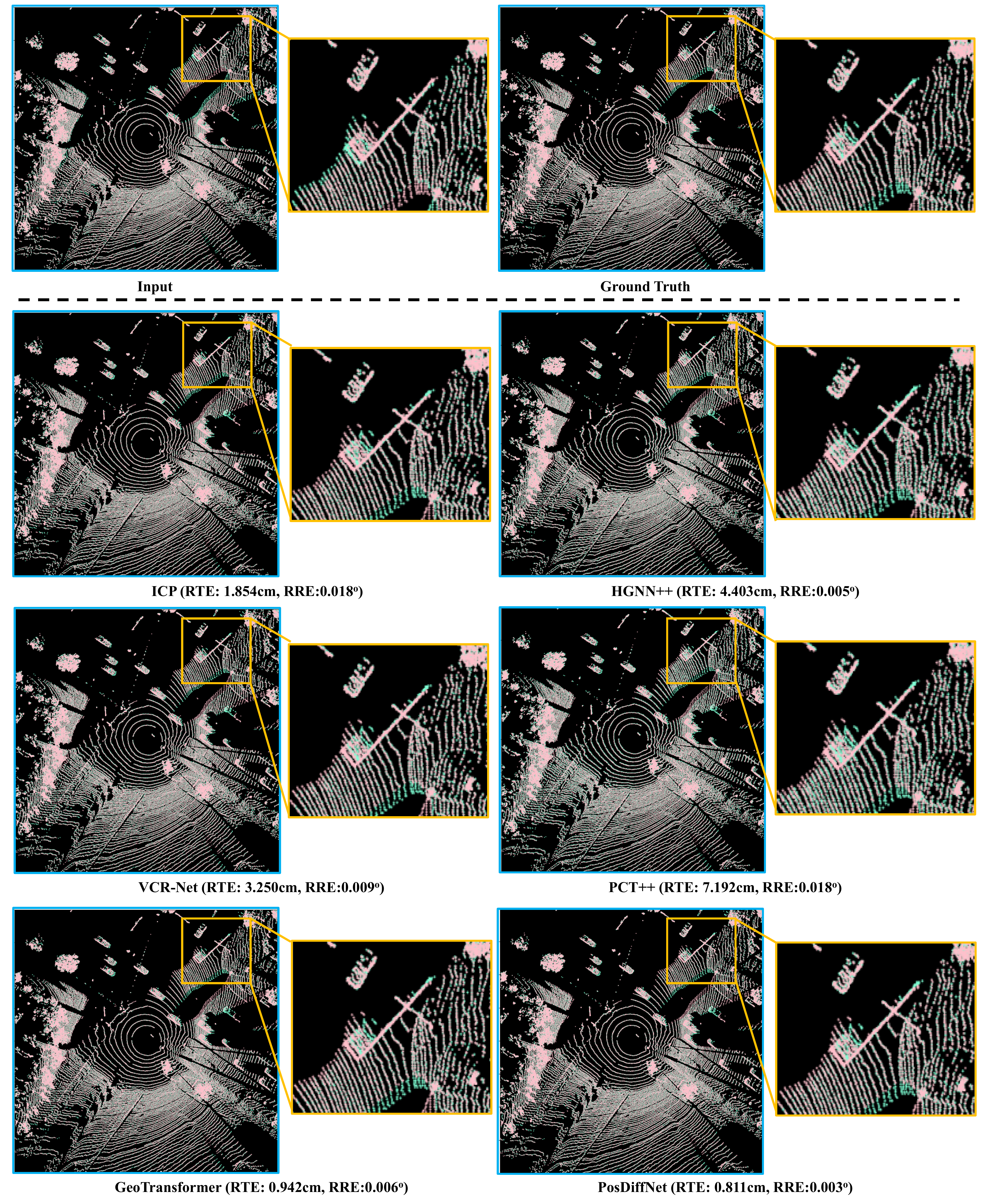}
\vspace{-0.3cm}
\caption{Point cloud registration examples of aligning the point cloud pairs with the predicted rotation and translation from the different methods using the Boreas dataset for training and testing.}
\label{fig:results_PCR_boreas} 
\end{figure*}

\begin{figure*}[!hbt]
\centering
\includegraphics[width=0.93\textwidth]{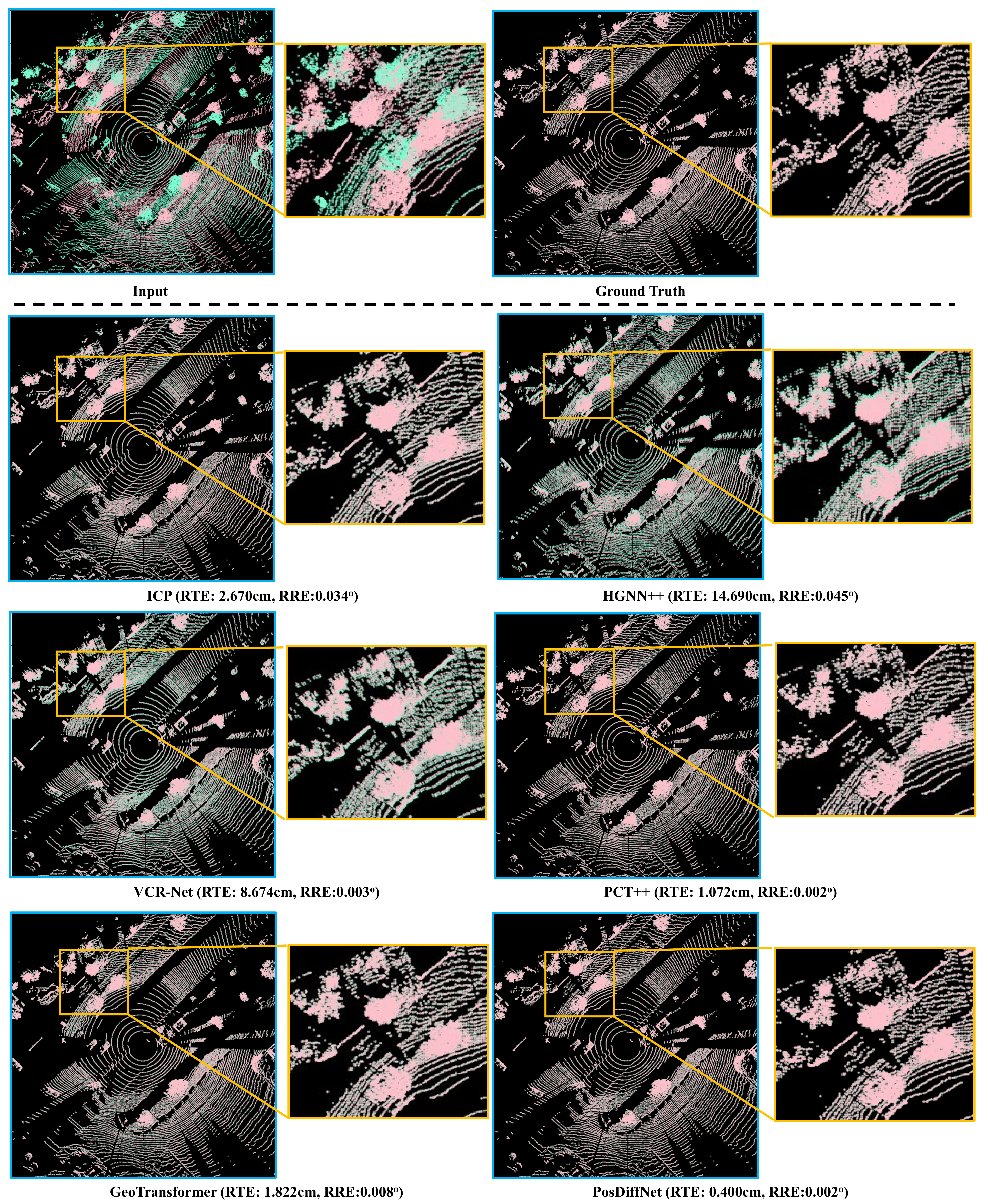}
\vspace{-0.3cm}
\caption{Point cloud registration examples of aligning the point cloud pairs with the predicted rotation and translation from the different methods using the synthetic Boreas dataset for training and testing.}
\label{fig:results_PCR_boreas_random} 
\end{figure*}

\cref{fig:results_PCR_boreas} and \cref{fig:results_PCR_boreas_random} illustrate multiple examples of point cloud alignment achieved using the predicted transformation, including both rotation and translation predictions. The quality of the predicted transformation can be directly assessed by examining the degree of overlap between the two frames. 

\section*{Broader Impact}
This work introduces positional neural diffusion into point cloud registration, addressing the challenges associated with large fields of view. Its implications span various sectors, including autonomous systems, 3D reconstruction and modeling, as well as environmental monitoring and geospatial analysis.
Specifically, our method for point cloud registration enables the development of more robust learning-based odometry for autonomous vehicles by leveraging neural diffusion for feature representation. This advancement plays a crucial role in enhancing 3D environment perception and localization capabilities.
Furthermore, in the field of environmental monitoring and geospatial analysis, our method offers the means to align point clouds obtained from different sensors, facilitating the creation of comprehensive and accurate 3D representations of both natural and built environments. This capability supports a wide range of applications, such as disaster management, land surveying, forestry, and climate change analysis.
To summarize, the broader impacts of our work are far-reaching within AI applications related to 3D point clouds.

\end{document}